\pdfoutput=1
\newif\ifdraft \draftfalse
\draftfalse
\def\focs{0} %Set this to 0 to revert to standard formatting; 1 to
\ifnum\focs=1
\documentclass[10pt, conference, compsocconf]{IEEEtran}
\IEEEoverridecommandlockouts
\else
\documentclass[11pt]{article}
\fi

\usepackage[utf8]{inputenc}
\usepackage[english]{babel}
\usepackage{xspace}
\usepackage{amsmath, amssymb, amsthm,amsfonts}
\usepackage{bbm}
\ifnum\focs=0
\usepackage{url}
\fi
\usepackage{verbatim}
\ifnum\focs=0
\usepackage{fullpage}
\fi
\usepackage[numbers]{natbib}
\usepackage{paralist}
\usepackage[linesnumbered]{algorithm2e}
\usepackage{color}
\ifnum\focs=0
\usepackage[]{hyperref}
\usepackage{bookmark}
\fi
\usepackage[capitalize]{cleveref}
\usepackage{ltxcmds}

\newtheorem{theorem}{Theorem}[section]
\newtheorem{defn}[theorem]{Definition}
\newtheorem{corollary}[theorem]{Corollary}
\newtheorem{lemma}[theorem]{Lemma}

\newtheorem{claim}[theorem]{Claim}

\title{Max-Information, Differential Privacy, and Post-Selection
  Hypothesis Testing\thanks{A preliminary version of this paper
    appeared in the proceedings of \emph{FOCS 2016} \cite{RogersRST16}.}}
\ifnum\focs=0
\author{Ryan Rogers\thanks{Department of Applied Mathematics and
    Computational Science, University of Pennsylvania. Email:
    \texttt{ryrogers@sas.upenn.edu}. Supported in part by a grant from
    the Sloan foundation and NSF grant CNS-1253345} \and Aaron
  Roth\thanks{Department of Computer and Information Sciences,
    University of Pennsylvania. Email:
    \texttt{aaroth@cis.upenn.edu}. Supported in part by a grant from
    the Sloan foundation, a Google Faculty Research Award, and NSF
    grants CNS-1513694 and CNS-1253345.} \and Adam
  Smith\thanks{Computer Science and Engineering Department,
   The  Pennsylvania State University. Email:
    \texttt{\{asmith,omthkkr\}@cse.psu.edu}. Supported in part by a grant from
    the Sloan foundation, a Google Faculty Research Award, and NSF
    grant IIS-1447700.} \and Om Thakkar\footnotemark[3] }
\else
\author{\IEEEauthorblockN{Ryan Rogers\IEEEauthorrefmark{1},
Aaron Roth\IEEEauthorrefmark{2},
Adam Smith\IEEEauthorrefmark{3} and
Om Thakkar\IEEEauthorrefmark{3}}
\IEEEauthorblockA{\IEEEauthorrefmark{1}Applied Math and Computational Sciences\\
University of Pennsylvania, Philadelphia, PA USA\\
E-mail: ryrogers@sas.upenn.edul}
\IEEEauthorblockA{\IEEEauthorrefmark{2}Computer and Information Science\\
University of Pennsylvania, Philadelphia, PA USA\\
E-mail: aaroth@cis.upenn.edu}
\IEEEauthorblockA{\IEEEauthorrefmark{3}Computer Science and Engineering Department\\
The Pennsylvania State University, University Park, PA USA\\
E-mail: \{asmith,omthkkr\}@cse.psu.edu}}

\fi

\newcommand{\eps}{\epsilon}
\newcommand\numberthis{\addtocounter{equation}{1}\tag{\theequation}}
\newcommand{\rynote}[1]{\ifdraft\textcolor{red}{[Ryan: #1]}\else\ignorespaces\fi}
\newcommand{\omnote}[1]{\ifdraft\textcolor{blue}{[Om: #1]}\else\ignorespaces\fi}
\newcommand{\arnote}[1]{\ifdraft\textcolor{cyan}{[Aaron: #1]}\else\ignorespaces\fi}

\newcommand\R{\mathbb{R}}
\newcommand{\cA}{\mathcal{A}}
\newcommand{\cB}{\mathcal{B}}
\newcommand{\cC}{\mathcal{C}}
\newcommand{\cD}{\mathcal{D}}
\newcommand{\cE}{\mathcal{E}}
\newcommand{\cF}{\mathcal{F}}
\newcommand{\cG}{\mathcal{G}}

\newcommand{\cM}{\mathcal{M}}

\newcommand{\cO}{\mathcal{O}}
\newcommand{\cP}{\mathcal{P}}

\newcommand{\cR}{\mathcal{R}}
\newcommand{\cS}{\mathcal{S}}
\newcommand{\cT}{\mathcal{T}}

\newcommand{\cX}{\mathcal{X}}
\newcommand{\cXax}{\cX(a,\bbx_1^{i-1})}
\newcommand{\cY}{\mathcal{Y}}
\newcommand{\cZ}{\mathcal{Z}}

\newcommand{\bC}{\mathbf{C}}
\newcommand{\bX}{\mathbf{X}}
\newcommand{\bY}{\mathbf{Y}}

\newcommand{\bba}{\mathbf{a}}
\newcommand{\bbb}{\mathbf{b}}
\newcommand{\bbc}{\mathbf{c}}

\newcommand{\bbx}{\mathbf{x}}
\newcommand{\bby}{\mathbf{y}}

\newcommand{\edi}{\approx_{\eps,\delta}}

\DeclareMathOperator*{\Expectation}{\mathbb{E}}
\newcommand{\Ex}[2]{\Expectation_{#1}\left[#2\right]}
\DeclareMathOperator*{\Probability}{\mathrm{Pr}}
\newcommand{\prob}[1]{\mathrm{Pr}\left[#1\right]}
\newcommand{\Prob}[2]{\Probability_{#1}\left[#2\right]}

\makeatletter

\renewcommand{\footnote}[2][\empty]{%
  \nolinebreak%
  \addtocounter{footnote}{+1}%
  \xdef\sfootnote@number{\arabic{footnote}}%
  \ltx@ifpackageloaded{hyperref}{% hyperref loaded
    \ifHy@hyperfootnotes% option hyperfootnotes=true
      \addtocounter{Hfootnote}{+1}%
        \global\let\Hy@saved@currentHref\@currentHref%
        \hyper@makecurrent{Hfootnote}%
        \global\let\Hy@footnote@currentHref\@currentHref%
        \global\let\@currentHref\Hy@saved@currentHref%
    \fi%
   }{% hyperref not loaded, nothing to be done here
   }%
  \xdef\sfootnote@opt{#1}% contains the optional argument
  \xdef\sfootnote@arabic{\arabic{footnote}}% is the Arabic footnotenumber
  \edef\sfootnote@formated{\thefootnote}% could also be * or dagger
  \ifx\sfootnote@opt\empty% i.e. no optional argument used
    \footnotetext{\label{fnr:\sfootnote@arabic}#2}%
  \else%
    \ltx@ifpackageloaded{hyperref}{% hyperref loaded
      \footnotetext[#1]{\phantomsection\label{fnr:\sfootnote@arabic}#2}%
     }{% hyperref not loaded
      \footnotetext[#1]{\label{fnr:\sfootnote@arabic}#2}%
     }%
  \fi%
  \ltx@ifpackageloaded{hyperref}{% hyperref package loaded
    \ifHy@hyperfootnotes% option hyperfootnotes=true
      \hbox {\@textsuperscript {\normalfont \ref{fnr:\sfootnote@arabic}}}%
    \else% option hyperfootnotes=false
      \hbox {\@textsuperscript {\normalfont \ref*{fnr:\sfootnote@arabic}}}%
    \fi%
  }{% hyperref package not loaded
    \hbox {\@textsuperscript {\normalfont \ref{fnr:\sfootnote@arabic}}}%
   }%
}
\begin{document}

\ifnum\focs=0
\begin{titlepage}
\thispagestyle{empty}
\fi

\maketitle

\begin{abstract}
In this paper, we initiate a principled study of how the
generalization properties of approximate differential privacy can be
used to perform adaptive hypothesis testing, while giving
statistically valid $p$-value corrections. We do this by observing
that the guarantees of algorithms with bounded approximate
\emph{max-information} are sufficient to correct the $p$-values of
adaptively chosen hypotheses, and then by proving that algorithms that
satisfy $(\epsilon,\delta)$-differential privacy have bounded
approximate max-information when their inputs are drawn from a product
distribution. 

This substantially extends the existing connection between differential privacy and max-information, which previously was only known to hold for (pure) $(\epsilon,0)$-differential privacy. It also extends our understanding of max-information as a partially unifying measure controlling the generalization properties of adaptive data analyses. We also show a lower bound, proving that (despite the strong composition properties of max-information), when data is drawn from a product distribution, $(\epsilon,\delta)$-differentially private algorithms can come \emph{first} in a composition with other algorithms satisfying max-information bounds, but not necessarily second if the composition is required  to  itself satisfy a nontrivial max-information bound. This, in particular, implies that the connection between $(\epsilon,\delta)$-differential privacy and max-information holds only for inputs drawn from particular distributions, unlike the connection between $(\epsilon,0)$-differential privacy and max-information.
\end{abstract}

\ifnum\focs=0
\newpage
\tableofcontents
\thispagestyle{empty}
\end{titlepage}
\fi
\section{Introduction}

\emph{Adaptive Data Analysis} refers to the reuse of data to perform analyses suggested by the outcomes of previously computed statistics on the same data. It is the common case when \emph{exploratory data analysis} and \emph{confirmatory data analysis} are mixed together, and both conducted on the same dataset. It models both well-defined, self-contained tasks, like selecting a subset of variables using the LASSO and then fitting a model to the selected variables, and also much harder-to-specify sequences of analyses, such as those that occur when the same dataset is shared and reused by multiple researchers.

Recently two lines of work have arisen, in statistics and computer science respectively, aimed at rigorous statistical understanding of adaptive data analysis. By and large, the goal in the statistical literature (often called ``selective'' or ``post-selection'' inference \cite{BBBZZ13}) is to derive valid hypothesis tests and tight confidence intervals around parameter values that arise from very specific analyses, such as LASSO model selection followed by least squares regression (see e.g. \cite{FST14,LSST13}). In contrast, the second line of work has aimed for generality (at the possible expense of giving tight application-specific bounds). This second literature imposes conditions on the algorithms performing each stage of the analysis, and makes no other assumptions on how, or in what sequence, the results are used by the data analyst.  Two algorithmic constraints that have recently been shown to guarantee that future analyses will be statistically valid are differential privacy \cite{DFHPRR15STOC,BNSSSU15} and bounded output description length, which are partially unified by a measure of information called \emph{max-information} \cite{DFHPRR15NIPS}. This paper falls into the second line of research---specifically, we extend the connection made in \cite{DFHPRR15STOC,BNSSSU15} between \emph{differential privacy} and the adaptive estimation of \emph{low-sensitivity queries} to a more general setting that includes  adaptive hypothesis testing with statistically valid $p$-values.

Our main technical contribution is a quantitatively tight connection between differential privacy and a \emph{max-information}. Max-information is a measure of correlation, similar to Shannon's mutual information, which allows bounding the change in the conditional probability of events relative to their a priori probability. Specifically, we extend a bound on the max-information of $(\eps,0)$-differentially private algorithms, due to \cite{DFHPRR15NIPS}, to the much larger class of $(\eps,\delta)$-differentially private algorithms.

\subsection{Post-Selection Hypothesis Testing}

To illustrate an application of our results, we consider  a simple model of one-sided hypothesis tests on real valued test statistics. Let $\cX$ denote a data domain.
A \emph{dataset} $\bbx$ consists of $n$ elements in $\cX$: $\bbx \in \cX^n$. A hypothesis test is defined by a \emph{test statistic} $\phi_i:\cX^n \rightarrow \mathbb{R}$, where we use $i$ to index different test statistics.
Given an output $a = \phi_i(\bbx)$, together with a distribution $\cP$ over the data domain, the $p$-value associated with $a$ and $\cP$ is simply the probability of observing a value of the test statistic that is at least as extreme as $a$, assuming the data was drawn independently from $\cP$: $p^{\cP}_i(a) = \Pr_{\bX \sim \cP^n}[\phi_i(\bX) \geq a]$.   Note that there may be multiple distributions $\cP$ over the data that induce the same distribution over the test statistic. With each test statistic $\phi_i$, we associate a \emph{null hypothesis} $H^{(i)}_0 \subseteq \Delta(\cX)$,\footnote{$\Delta(\cX)$ denotes the set of probability distributions over $\cX$.} which is simply a collection of such distributions. The $p$-values are always computed with respect to a distribution $\cP \in H^{(i)}_0$, and hence from now on, we elide the dependence on $\cP$ and simply write $p_i(a)$ to denote the $p$-value of a test statistic $\phi_i$ evaluated at $a$.

The goal of a hypothesis test is to \emph{reject the null hypothesis} if the data is not likely to have been generated from the proposed model, that is, if the underlying distribution from which the data were drawn was not  in $H^{(i)}_0$.  By definition, if $\bX$ truly is drawn from $\cP^n$ for some $\cP \in H^{(i)}_0$, then $p_i(\phi_i(\bX))$ is uniformly distributed over $[0,1]$.  A standard approach to hypothesis testing is to pick  a \emph{significance level} $\alpha \in [0,1]$ (often $\alpha = 0.05$), compute the value of the test statistic $a = \phi_i(\bX)$, and then \emph{reject} the null hypothesis if $p_i(a) \leq \alpha$. Under this procedure, the probability of incorrectly rejecting the null hypothesis---i.e., of rejecting the null hypothesis when $\bX \sim \cP^n$ for some $\cP \in H^{(i)}_0$---is at most $\alpha$.
An incorrect rejection of the null hypothesis is called a \emph{false discovery}.

The discussion so far presupposes that $\phi_i$, the test statistic in question, was chosen independently of the dataset $\bX$. Let $\cT$ denote a collection of test statistics, and suppose that we select a test statistic using a data-dependent selection procedure $\cA:\cX^n\rightarrow \cT$. If $\phi_i = \cA(\bX)$, then rejecting the null hypothesis when $p_i(\phi_i(\bX)) \leq \alpha$ may result in a false discovery with probability much larger than $\alpha$ (indeed, this kind of naive approach to \emph{post-selection} inference is suspected to be a primary culprit behind the prevalence of false discovery in empirical science \cite{GL14,WL16,SNS11}). This is because even if the null hypothesis is true ($\bX \sim \cP^n$ for some $\cP \in H^{(i)}_0$), the  distribution on $\bX$ \emph{conditioned on $\phi_i = \cA(\bX)$ having been selected} need not be $\cP^n$. Our goal in studying valid post-selection hypothesis testing is to find a \emph{valid} $p$-value correction function $\gamma:[0,1]\rightarrow [0,1]$, which we define as follows:

\begin{defn}[Valid $p$-value Correction Function] \label{defn:p_correct}
A function $\gamma:[0,1]\rightarrow [0,1]$ is a \emph{valid $p$-value correction function} for a selection procedure $\cA:\cX^n\rightarrow \cT$ if for every significance level $\alpha \in [0,1]$, the procedure:
\begin{compactenum}
\item Select a test statistic $\phi_i =\cA(\bX)$ using selection procedure $\cA$.
\item Reject the null hypothesis $H^{(i)}_0$ if $p_i(\phi_i(\bX)) \leq \gamma(\alpha)$.
\end{compactenum}
has probability at most $\alpha$ of resulting in a false discovery.
\end{defn}

Necessarily, to give a nontrivial correction function $\gamma$, we will need to assume that the selection procedure $\cA$ satisfies some useful property. In this paper, we focus on \emph{differential privacy}, which is a measure of algorithmic stability, and more generally, max-information, which is defined in the next subsection. Differential privacy is of particular interest because it is closed under post-processing and satisfies strong composition properties. This means that, if the test statistics in $\cT$ are themselves differentially private,
%(or satisfy bounded max-information),
then the selection procedure $\cA$ can represent the decisions of a worst-case data analyst, who chooses which hypothesis tests to run in arbitrary ways as a function of the outcomes of previously selected tests.

Finally, we note that, despite the fact that previous works \cite{DFHPRR15STOC,BNSSSU15} are explicitly motivated by the problem of false discovery in empirical science, most of the technical results to date have been about estimating the means of adaptively chosen predicates on the data (i.e., answering \emph{statistical queries}) \cite{DFHPRR15STOC}, and more generally, estimating the values of low-sensitivity (i.e., Lipschitz continuous) functions on the dataset \cite{BNSSSU15,RZ15,WLF16}. These kinds of results do not apply to the problem of adaptively performing hypothesis tests while generating statistically valid $p$-values, because $p$-values are by definition not low-sensitivity statistics. See \ifnum\focs=1 the full version \else Appendix \ref{sec:sens_p} \fi for a detailed discussion.

There is one constraint on the selection procedure $\cA$ that does allow us to give nontrivial $p$-value corrections---that $\cA$ should have bounded max-information.  A condition of bounded mutual information has also been considered \cite{RZ15} to give $p$-value corrections - but as we discuss in \ifnum\focs=1 the full version\else Appendix \ref{sect:mutual-max} \fi, it is possible to obtain a strictly stronger guarantee by instead reasoning via max-information.
Max-information is a measure introduced by \cite{DFHPRR15NIPS}, which we discuss next.

\subsection{Max-Information (and $p$-values)}
Given two (arbitrarily correlated) random variables $X$, $Z$, we let
$X\otimes Z$ denote a random variable (in a different probability
space) obtained by drawing independent copies of $X$ and $Z$ from
their respective marginal distributions. We write $\log$ to denote logarithms base 2.

\begin{defn}[Max-Information \cite{DFHPRR15NIPS}]
Let $X$ and $Z$ be jointly distributed random variables over the domain $(\mathcal{X},\mathcal{Z})$. The max-information between $X$ and $Z$, denoted by $I_\infty (X;Z)$, is the minimal value of $k$ such that for every $x$ in the support of $X$ and $z$ in the support of $Z$, we have $\prob{X=x | Z=z} \leq 2^k \prob{X=x}$. Alternatively, $$I_\infty (X;Z) = \log \sup\limits_{(x,z) \in (\mathcal{X},\mathcal{Z})} \dfrac{\prob{(X,Z) = (x,z)} }{\prob{X \otimes Z = (x,z)} }.$$
The $\beta$-approximate max-information between $X$ and $Z$ is defined as $$I_\infty^\beta (X;Z) = \log \sup\limits_{\substack{\mathcal{O} \subseteq (\mathcal{X} \times \mathcal{Z}),\\ \prob{ (X,Z) \in \mathcal{O} } > \beta}} \dfrac{\prob{(X,Z) \in \mathcal{O}} - \beta}{\prob{X\otimes Z \in \mathcal{O}} }.$$
We   say that an algorithm $\cA: \cX^n \to \cY$ has $\beta$-approximate max-information of $k$, denoted as $I^\beta_\infty(\cA,n) \leq k$, if for every distribution $\cS$ over elements of $\cX^n$, we have $I^\beta_\infty(\bX;\cA(\bX)) \leq k$ when $\bX\sim \cS$. We say that an algorithm $\cA: \cX^n \to \cY$ has $\beta$-approximate max-information of $k$ over \emph{product distributions}, written $I^\beta_{\infty,P}(\cA, n) \leq k$, if for every distribution $\cP$ over $\cX$, we have $I^{\beta}_\infty(\bX; \cA(\bX)) \leq k$ when $\bX\sim \cP^n$.
\label{defn:maxinfo}
\end{defn}

It follows immediately from the definition  that if an algorithm has bounded max-information, then we can control the probability of ``bad events'' that arise as a result of the dependence of $\cA(\bX)$ on $\bX$: for every event $\cO$, we have  $\Pr[(\bX, \cA(\bX)) \in \cO] \leq 2^k\Pr[\bX \otimes \cA(\bX) \in \cO]+\beta.$
%
% \begin{theorem}[Max-Information Controls Dependence \cite{DFHPRR15NIPS}]
% \label{thm:badevents}
% Let $\bX$ be a random dataset drawn according to a distribution $\cS \in \Delta(\cX^n)$. Let algorithm $\cA:\cX^n\rightarrow \cY$ be such that for some $\beta \geq 0$, $I^\beta_\infty(\cA,n) \leq k$. Then for any event $\cO \subseteq \cX^n \times \cY$:
% $$\Pr[(\bX, \cA(\bX)) \in \cO] \leq 2^k\Pr[\bX \otimes \cA(\bX) \in \cO]+\beta.$$
%  If $\cS$ is a product distribution (i.e., $\cS = \cP^n$ for some $\cP \in \Delta(\cX)$), then the same conclusion holds under the hypothesis that $I^\beta_{\infty,P}(\cA, n) \leq k$.
% \end{theorem}
For example,  if $\cA$ is a data-dependent selection procedure for selecting a test statistic, we can derive a valid $p$-value correction function $\gamma$ as a function of a max-information bound on $\cA$:
\begin{theorem}
\label{thm:maxinfo-pvalues}
Let $\cA:\cX^n\rightarrow \cT$ be a data-dependent algorithm for selecting a test statistic such that $I^\beta_{\infty,P}(\cA, n) \leq k$. Then the following function $\gamma$ is a valid $p$-value correction function for $\cA$:
$$\gamma(\alpha) = \max\left(\frac{\alpha - \beta}{2^k},0\right).$$
\end{theorem}
\begin{proof}
Fix a distribution $\cP^n$ from which the dataset $\bX$ is drawn. If $\frac{\alpha - \beta}{2^k} \leq 0$, then the theorem is trivial, so assume otherwise.
Define $\cO \subset \cX^n \times \cT$ to be the event that $\cA$ selects a test statistic for which the null hypothesis is true, but its $p$-value is at most $\gamma(\alpha)$:
$$\cO = \{(\bbx, \phi_i) : \cP \in H_0^{(i)} \textrm{ and }p_i(\phi_i(\bbx)) \leq \gamma(\alpha)\}$$
Note that the event $\cO$ represents exactly those outcomes for which using $\gamma$ as a $p$-value correction function results in a false discovery.
Note also that, by definition of the null hypothesis, $\Pr[\bX \otimes \cA(\bX) \in \cO] \leq \gamma(\alpha) = \frac{\alpha-\beta}{2^k}$. Hence, by the guarantee that $I^\beta_{\infty,P}(\cA, n) \leq k$, we have that
$\Pr[(\bX, \cA(\bX) \in \cO)] $ is at most $2^k\cdot\left(\frac{\alpha-\beta}{2^k}\right) + \beta = \alpha$.
\end{proof}

Because of Theorem \ref{thm:maxinfo-pvalues}, we are interested in methods for usefully selecting test statistics using data dependent algorithms $\cA$ for which we can bound their max-information. It was shown in \cite{DFHPRR15NIPS} that algorithms which satisfy \emph{pure} differential privacy also have a guarantee of bounded max-information:
\begin{theorem}[Pure Differential Privacy and Max-Information \cite{DFHPRR15NIPS}] \label{thm:pureprivacyinfo}
Let $\cA:\cX^n\rightarrow \cY$ be an $(\epsilon,0)$-differentially private algorithm. Then for every $\beta \geq 0$:
\ifnum\focs=1
\begin{align*}
I_{\infty}(\cA, n) &\leq \log(e)\cdot \epsilon n \text{, and}\\
\\ I^\beta_{\infty,P}(\cA, n) &\leq \log(e)\cdot\left(\epsilon^2 n/2+\epsilon\sqrt{n\ln(2/\beta)/2}\right)
\end{align*}
\else
\[
I_{\infty}(\cA, n) \leq \log(e)\cdot \epsilon n \ \ \ \ \mathrm{and} \ \ \ \ I^\beta_{\infty,P}(\cA, n) \leq \log(e)\cdot\left(\epsilon^2 n/2+\epsilon\sqrt{n\ln(2/\beta)/2}\right)
\]
\fi

\end{theorem}

This connection is powerful, because there are a vast collection of data analyses  for which we have  differentially private algorithms--- including a growing literature on differentially private hypothesis tests  \cite{JS13,USF13,YFSU14,KS16,DSZ15,Shef15,WLK15,GLRV16}. However, there is an important gap: Theorem \ref{thm:pureprivacyinfo} holds only for \emph{pure} $(\epsilon,0)$-differential privacy, and not for approximate $(\epsilon,\delta)$-differential privacy. Many statistical analyses can be performed much more accurately subject to approximate differential privacy, and it can be easier to analyze private hypothesis tests that satisfy approximate differential privacy, because the approximate privacy constraint is amenable to perturbations using Gaussian noise (rather than Laplace noise) \cite{GLRV16}. Most importantly, for pure differential privacy, the privacy parameter $\epsilon$ degrades \emph{linearly} with the number of analyses performed, whereas for approximate differential privacy, $\epsilon$ need only degrade with the \emph{square root} of the number of analyses performed \cite{DRV10}. Hence, if the connection between max-information and differential privacy held also for approximate differential privacy, it would be possible to perform quadratically more adaptively chosen statistical tests without requiring a larger $p$-value correction factor.
\subsection{Our Results}
In addition to the framework just described for reasoning about adaptive hypothesis testing, our main technical contribution is to extend the connection between differential privacy and max-information to approximate differential privacy. We show the following (see Section \ref{sec: max_info} for a complete statement):
\newline\newline\textbf{Theorem \ref{thm:main} (Informal).} \textit{Let $\mathcal{A}: \cX^n \to \cY$ be an $(\eps,\delta)$-differentially
private algorithm.
Then,}
$$
I^\beta_{\infty,P}(\cA, n) = O\left(n\eps^2    + n \sqrt{\tfrac{\delta}{\eps}} \right) \quad
\text{ for } \beta =   O\left(n \sqrt{\tfrac{\delta}{\eps}}\right).%e^{-\eps^2 n} +
$$
\newline
It is worth noting several things. First, this bound nearly matches the bound for max-information over product distributions from Theorem \ref{thm:pureprivacyinfo}, except Theorem \ref{thm:main} extends the connection to the substantially more powerful class of $(\epsilon,\delta)$-differentially private algorithms. The bound is qualitatively tight in the sense that despite its generality, it can be used to nearly recover the tight bound on the generalization properties of differentially private mechanisms for answering low-sensitivity queries that was proven using a specialized analysis in \cite{BNSSSU15} (see \ifnum\focs=1 the full version \else Appendix \ref{app:compare} \fi for a comparison).

We also only prove a bound on the max-information for product distributions on the input, and not for all distributions (that is, we bound $I^\beta_{\infty,P}(\cA, n)$ and not $I^\beta_{\infty}(\cA, n)$). A bound for general distributions would be desirable, since such bounds compose gracefully  \cite{DFHPRR15NIPS}. Unfortunately, a bound for general distributions based solely on $(\eps,\delta)$-differential privacy is impossible: a construction of De \cite{De12} implies the existence of $(\epsilon,\delta)$-differentially private algorithms for which the max-information between input and output on arbitrary distributions is much larger than the bound in Theorem \ref{thm:main}.

One might nevertheless hope that bounds on the max-information under product distributions can be meaningfully composed. Our second main contribution is a negative result, showing that such bounds do not compose when algorithms are selected adaptively.
Specifically, we analyze the adaptive composition of two algorithms, the first of which has a small finite range (and hence, by \cite{DFHPRR15NIPS}, small bounded max-information), and the second of which is $(\epsilon,\delta)$-differentially private. We show that the composition of the two algorithms can be used to exactly recover the input dataset, and hence, the composition does not satisfy any nontrivial max-information bound.

%However, because it was shown in \cite{DFHPRR15NIPS} (also given in the appendix) that max-information satisfies a strong composition property---that if two algorithms $\cA$ and $\cB$ satisfy $I_{\infty}^{\beta_1}(\cA, n) \leq k_1$ and $I_{\infty}^{\beta_2}(\cB, n) \leq k_2$, then their composition $\cC$ satisfies $I_{\infty}^{\beta_1+\beta_2}(\cC, n) \leq k_1+k_2$---this implies that $(\epsilon,\delta)$-differentially private algorithms cannot satisfy any nontrivial max-information bound.

\subsubsection{Further Interpretation}
Although our presentation thus far has been motivated by $p$-values, an algorithm $\cA$ with bounded max-information allows a data analyst to treat \emph{any event} $\cA(\bbx)$ that is a function of the output of the algorithm ``as if'' it is independent of the dataset $\bbx$, up to a correction factor determined by the max-information bound. Our results thus substantially broaden the class of analyses for which approximate differential privacy promises generalization guarantees---this class was previously limited to estimating the values of low-sensitivity numeric valued queries (and more generally, the outcomes of low-sensitivity optimization problems) \cite{BNSSSU15}.

Our result also further develops the extent to which max-information can be viewed as a unifying information theoretic measure controlling the generalization properties of adaptive data analysis. Dwork et al. \cite{DFHPRR15NIPS} previously showed that algorithms that satisfy bounded output description length, and algorithms that satisfy pure differential privacy (two constraints known individually to imply adaptive generalization guarantees), both have bounded max-information. Because bounded max-information satisfies strong composition properties, this connection implies that algorithms with bounded output description length and pure differentially private algorithms can be composed in arbitrary order and the resulting composition will still have strong generalization properties. Our result brings approximate differential privacy partially into this unifying framework.  In particular, \emph{when the data is drawn from a product distribution}, if an analysis that starts with an (arbitrary) approximate differentially private computation is followed by an arbitrary composition of algorithms with bounded max-information, then the resulting composition will satisfy a max-information bound. However, unlike with compositions consisting solely of bounded description length mechanisms and pure differentially private mechanisms, which can be composed in arbitrary order, in this case \emph{it is important that the approximate differentially private computation come first}. This is because, even if the dataset $\bbx$ is initially drawn from a product distribution, the conditional distribution on the data that results after observing the outcome of an initial computation need not be a product distribution any longer. In fact, the lower bound we prove in Section \ref{sec:lowerbound} is an explicit construction in which the composition of a bounded description length algorithm, followed by an approximate differentially private algorithm can be used to exactly reconstruct a dataset drawn from a product distribution (which can in turn be used to arbitrarily overfit that dataset).

Finally, we draw a connection between max-information and mutual information that allows us to improve on several prior results that dealt with mutual information \ifnum\focs=1 \cite{RZ15,MMPRTV11} \else \cite{MMPRTV11,RZ15}\fi.
 \rynote{Modified}
 We present the proof in \ifnum\focs=1 the full version. \else \Cref{sect:mutual-max}. \fi

\begin{lemma}\label{lem:mutual-max}
Let $\bX,\bY$ be a pair of discrete random variables defined on the same probability space, where $\bX$ takes values in a finite set $\Sigma$.% $\cX^n$.
\begin{itemize}
\item If $I(\bX;\bY) \leq m$ then, for every $k >0$, we have $I_\infty^{\beta(k)}(\bX;\bY) \leq k$ where $\beta(k) = \tfrac{m + 0.54}{k}$.
\item If $I_\infty^{\beta}(\bX;\bY) \leq k$ for $k > 0$ and $0\leq \beta \leq \frac{3(1-2^{-k})}{20}$, then\footnote{The bound stated here is looser than that claimed in the conference version~\cite{RogersRST16}; we do not know a proof of the original claim.}
  $I(\bX;\bY) \leq 2k \ln(2)+
%\frac{2\beta n \log|\cX|}{1- 2^{-k} } + \frac{2\beta}{1- 2^{-k} } \ln\left(\frac{1- 2^{-k} }{2\beta}\right)$. 
\dfrac{2\beta  \log(|\Sigma|/ 2\beta)}{1- 2^{-k}  }$.
\end{itemize}
\end{lemma}

% \begin{lemma}\label{lem:mutual-max}
% Let $\cA: \cX^n \to \cT$ be a selection rule.
% \begin{itemize}
% \item If $I(\bX;\cA(\bX)) \leq m$ and $\bX \sim \cS$ for any distribution over $\cX^n$, then for any $k >0$, $I_\infty^{\beta(k)}(\bX;\cA(\bX)) \leq k$ for $\beta(k) \leq \tfrac{m + 0.54}{k}$.
% \item If $I_\infty^{\beta}(\bX;\cA(\bX)) \leq k$ for $k > 0$ and $\beta \in \left[0,\frac{3(1-2^{-k})}{20}\right]$, then $I(\bX;\cA(\bX)) \leq 2k \ln(2)+\frac{2\beta n \log|\cX|}{1- 2^{-k} } + \frac{2\beta}{1- 2^{-k} } \ln\left(\frac{1- 2^{-k} }{2\beta}\right)$.
% \end{itemize}
% \end{lemma}
%%%

We are able to improve on the $p$-value correction function implicitly given in \cite{RZ15} given a mutual information bound, by first converting mutual information to a max-information bound and applying the $p$-value correction function from this paper. Our main theorem \Cref{thm:main} combined with \Cref{lem:mutual-max} also obtains an improved bound on the mutual information of approximate differentially private mechanisms from Proposition 4.4 in \cite{MMPRTV11}\footnote{Thanks to Salil Vadhan for pointing out to us this implication.}.  The following corollary improves the bound from \cite{MMPRTV11} in its dependence on $|\cX|$ from $|\cX|^2 \cdot \log(1/|\cX|)$ to $\log|\cX|$.
\begin{corollary}
Let $\cA:\cX^n \to \cT$ be $(\eps,\delta)$-differentially private and $\bX \sim \cP^n$.  If\footnote{The parameter ranges stated here are different than those claimed in the conference version~\cite{RogersRST16}; this modification is a result of the change in \Cref{lem:mutual-max}.} $\epsilon \in (0,1/2], \eps = \Omega \left( \frac{1}{\sqrt{n}} \right),$ and $\delta = O(\tfrac{\epsilon}{n^2})$, we then have:
\begin{align*}
\ifnum\focs=1 & I(\bX;\cA(\bX)) \\ & =  O \left(n\epsilon^2 +  n \sqrt{\tfrac{\delta}{\epsilon}} \left( 1 + \ln\left( \tfrac{1}{n} \sqrt{\tfrac{\epsilon}{\delta}}\right) + n\log|\cX| \right) \right).
\else
I(\bX;\cA(\bX)) =  O \left(n\epsilon^2 +  n \sqrt{\tfrac{\delta}{\epsilon}} \left( 1 + \ln\left( \tfrac{1}{n} \sqrt{\tfrac{\epsilon}{\delta}}\right) + n\log|\cX| \right) \right).
\fi
\end{align*}
\end{corollary}
\rynote{end modifications}

\subsection{Other Related Work}
Differential Privacy is an algorithmic stability condition introduced by Dwork et al. \cite{DMNS06}. Its connection to adaptive data analysis was made by Dwork et al. \cite{DFHPRR15STOC} and both strengthened and generalized by Bassily et al. \cite{BNSSSU15}. Dwork et al. \cite{DFHPRR15NIPS} showed that algorithms with bounded description length outputs have similar guarantees for adaptive data analysis, and introduced the notion of max-information. Cummings et al. \cite{CLNRW16} give a third method---compression schemes---which can also guarantee validity in adaptive data analysis in the context of learning. Computational and information theoretic lower bounds for adaptively estimating means in this framework were proven by Hardt and Ullman \cite{HU14}, and Steinke and Ullman \cite{SU15}.

Russo and Zou \cite{RZ15} show how to bound the \emph{bias} of sub-gaussian statistics selected in a data-dependent manner, in terms of the mutual information between the selection procedure and the value of the statistics. In particular (using our terminology), they show how to give a valid $p$-value correction function in terms of this mutual information. In
\ifnum\focs=1 the full version \else Appendix \ref{sect:mutual-max}\fi, we demonstrate that if a bound on the mutual information between the dataset and the output of the selection procedure is known, then it is possible to substantially improve on the $p$-value correction function given by \cite{RZ15} by instead using the mutual information bound to prove a max-information bound on the selection procedure. \cite{WLF16} study adaptive data analysis in a similar framework to \cite{RZ15}, and give a minimax analysis in a restricted setting.

McGregor et al. \cite{MMPRTV11}, and De \cite{De12} also study (among other things) information theoretic bounds satisfied by differentially private algorithms. Together, they prove a result that is analogous to ours, for \emph{mutual information}---that while pure differentially private algorithms have bounded mutual information between their inputs and their outputs, a similar bound holds for (approximate) $(\epsilon,\delta)$-differentially private algorithms only if the data is drawn from a product distribution. We improve quantitatively on this bound.

\section{Preliminaries}
We will use the following vector notation throughout:
$\bbx = (x_1, \cdots, x_n)$, $\bbx_a^b = (x_a,x_{a+1}, \cdots,x_b)$, $  (\bbx_{-i},t) = (x_1,\cdots, x_{i-1}, t, x_{i+1}, \cdots, x_n).$
We denote the distribution of a random variable $X$ as $p(X)$.
In our analysis, jointly distributed random variables $(X,Z)$ will typically be of the form $(\bX,\cA(\bX))$ where $\bX \sim \cP^n$ is a dataset of $n$ elements sampled from domain $\cX$, and $\cA: \cX^n \to \cY$ is a (randomized) algorithm that maps a dataset to some range $\cY$.  We denote by $\cA(\bX)$ the random variable that results when $\cA$ is applied to a dataset $\bX \sim \cP^n$ (note that here, the randomness is both over the choice of dataset, and the internal coins of the algorithm).  When the input variable is understood, we will sometimes simply write $\cA$.

It will be useful in our analysis to compare the distributions of two random variables.  In the introduction, we define (approximate-) max-information, and we now give some other measures between distributions.  We first define indistinguishability, and then differential privacy.

\begin{defn}[Indistinguishability \cite{KS14}]
Two random variables $X,Y$ taking values in a set $\cD$ are $(\eps, \delta)$-indistinguishable, denoted $X \edi Y$, if for all $\mathcal{O} \subseteq \cD$,
\begin{align*}
\Prob{}{X \in \mathcal{O}} &\leq e^\eps \cdot\Prob{}{Y \in \mathcal{O}} + \delta \text{ and }\\ \Prob{}{Y \in \mathcal{O}} &\leq e^\eps \cdot\Prob{}{X \in \mathcal{O}} + \delta .
\end{align*}

%$$ \sup\limits_{\mathcal{O} \subseteq D} \log \dfrac{\Pr (X \in \mathcal{O}) - \delta}{\Pr (Y \in \mathcal{O})}\leq \eps \log (e)  \text{ and } \sup\limits_{\mathcal{O} \subseteq D} \log \dfrac{\Pr (Y \in \mathcal{O}) - \delta}{\Pr (X \in \mathcal{O})}\leq \eps \log (e). $$
\end{defn}

\begin{defn}[Point-wise indistinguishibility \cite{KS14}]
Two random variables $X,Z$ taking values in a set $\cD$ are point-wise $(\eps, \delta)$-indistinguishable if with probability at least $1-\delta$ over $ a \sim p(X)$:
$$
e^{-\epsilon} \Prob{}{Z = a} \leq \Prob{}{X = a} \leq e^{\epsilon} \Prob{}{Z=a}.
$$
\end{defn}

Before we define differential privacy, we say that two databases $\bbx,\bbx' \in \cX^n$ are \emph{neighboring} if they differ in at most one entry.  We now define differential privacy in terms of indistinguishability:

\begin{defn}[Differential Privacy \cite{DMNS06,DKMMN06}]
A randomized algorithm $\mathcal{A}: \cX^n \to \cY$ is $(\eps, \delta)$-differentially private if for all neighboring datasets $\bbx, \bbx' \in \cX^n$, we have $\cA(\bbx) \edi \cA(\bbx')$.
\end{defn}

In the appendix, we give several useful connections between these definitions along with other more widely known measures between distributions, e.g., KL-divergence, and total-variation distance.
%%% Local Variables:
%%% mode: latex
%%% TeX-master: "main"
%%% End:

\section{Max-Information for $(\eps,\delta)$-Differentially Private Algorithms} \label{sec: max_info}
In this section, we prove a bound on approximate max-information for $(\eps,\delta)$-differentially private algorithms over product distributions.

\begin{theorem}\label{thm:main}
Let $\mathcal{A}: \cX^n \to \cY$ be an $(\eps,\delta)$-differentially
private algorithm for $\eps \in (0,1/2]$ and $\delta \in \left (0,\eps \right )$.
%and $\delta = \Omega(\eps \cdot e^{-\eps^2 n})$.
%For every distribution $\mathcal{P}$ over  $\mathcal{X}$, if $\bX \sim \cP^n$, we have that, 
For $\beta
%\leq e^{-\eps^2 n} + n \left ( \dfrac{8}{3}\sqrt{\dfrac{\delta}{\eps}} + 2 \delta \left ( \dfrac{1}{\eps} + 2 \right ) + 4 \sqrt{\eps \delta} \right )
=  e^{-\eps^2 n} + O\left(n \sqrt{\tfrac{\delta}{\eps}}\right)$, we have
\begin{align*}
I^\beta_{\infty,P}(\cA,n)
%I^{\beta }_{\infty}(\bX;\cA(\bX))
%& \leq \eps^2 n \left ( 6 \sqrt{2}  + \left ( \dfrac{36 + 72 \sqrt{\eps\delta}}{1 + 3\eps - \sqrt{\eps\delta}(6 \eps + 2)} \right ) \right ) + n\sqrt{\eps\delta} \left ( \log e \left (2 + 6\eps \right ) + \dfrac{12 \eps}{1 + 3\eps - \sqrt{\eps\delta}(6 \eps + 2)}  \right )  \\
%& \qquad - \tfrac{\delta n  \log e}{6} \left ( 8 + 288 \eps \right)  \\
& = O\left(\eps^2 n   + n \sqrt{\tfrac{\delta}{\epsilon}} \right).
\end{align*}
%In other words, $I^\beta_{\infty,P}(\cA,n) = O\left(\eps^2 n   + n \sqrt{\eps \delta} \right)$ for $\beta = e^{-\eps^2 n} + O\left(n \sqrt{\tfrac{\delta}{\eps}}\right)$. \omnote{Added statement here. Feel free to remove if not required.}.
\end{theorem}

We will prove \Cref{thm:main} over the course of this section, using a number of lemmas. We first set up some notation.  We will sometimes abbreviate
conditional probabilities of the form $\Prob{}{\bX=\bbx|\cA = a}$ as
$\Prob{}{\bX=\bbx|a}$ when the random variables are clear from
context.  Further, for any $\bbx \in \cX^n$ and $ a \in \cY$, we define

\begin{align*}
Z(a,\bbx) & \stackrel{\text{def}}{=} \log \left( \dfrac{\Prob{}{\cA=a
            , \bX=\bbx}}{\Prob{}{\cA = a} \cdot \Prob{}{\bX=\bbx}}
            \right) \\
 & = \sum\limits_{i=1}^n \log \left( \dfrac{\Prob{}{X_i = x_i | a, \bbx_1^{i-1}}}{\Prob{}{X_i = x_i}} \right) \numberthis \label{eqn:ln_sum}
\end{align*}

If we can bound $Z(a,\bbx)$ with high probability over $(a,\bbx) \sim p(\cA(\bX),\bX)$, then we can bound the approximate max-information by using the following lemma:
\begin{lemma}[{\cite[Lemma 18]{DFHPRR15NIPS}}] If
$\Prob{}{Z(\cA(\bX),\bX) \geq k} \leq \beta$, then $I_\infty^\beta(\cA(\bX);\bX) \leq k$.
\label{lem:boundmaxinfo}
\end{lemma}

We next define each term in the sum of $Z(a,\bbx)$ as
\begin{equation}
Z_i(a,\bbx_1^{i}) \stackrel{\text{def}}{=} \log \dfrac{\Prob{}{X_i = x_i | a, \bbx_1^{i-1}}}{\Prob{}{X_i = x_i}}. \label{eqn:Z_i}
\end{equation}

The plan of the proof is simple: our goal is to apply Azuma's inequality \ifnum\focs=0 (\Cref{thm:azuma})\fi to the sum of the $Z_i$'s to achieve a bound on $Z$ with high probability. Applying Azuma's inequality requires both understanding the expectation of each term $Z_i(a,\bbx_1^{i})$, and being able to argue that each term is bounded. Unfortunately, in our case, the terms are not always bounded -- however, we will be able to show that they are bounded with high probability. This plan is somewhat complicated by the conditioning in the definition of $Z_i(a,\bbx_1^{i})$.

First, we argue that we can bound each $Z_i$ with high probability. This argument takes place over the course of Claims \ref{claim:bounded1}, \ref{claim:Ei}, \ref{claim:Fi} and \ref{claim:Gi}.
\begin{claim}
\label{claim:bounded1}
If $\cA$ is $(\eps,\delta)$-differentially private and $\bX \sim \cP^n$, then for each $i \in [n]$ and each prefix $\bbx_1^{i-1} \in \cX^{i-1}$, we have:
$$
(\cA,X_i)|_{\bbx_1^{i-1}} \edi \cA|_{\bbx_1^{i-1}} \otimes X_i .
$$
\end{claim}
\ifnum\focs=0
\begin{proof}
Fix any set $\cO \subseteq \cY \times \cX$ and prefix $\bbx_1^{i-1} \in \cX^{i-1}$. We then define the set $\cO_{x_i} = \{a \in \cY : (a,x_i) \in \cO \}$.  Now, we have that:
\begin{align*}
& \prob{(\cA(\bX),X_i )  \in\cO  | \bbx_1^{i-1}}  = \sum_{x_i \in \cX}\prob{X_i = x_i} \prob{\cA(\bX) \in \cO_{x_i}  | \bbx_1^{i-1}, x_i} \\
& \qquad \leq \sum_{x_i \in \cX}\prob{X_i = x_i}\left( e^{\eps} \prob{\cA(\bX) \in \cO_{x_i} | \bbx_1^{i-1}, t_i} + \delta \right) \qquad \forall t_i \in \cX
\end{align*}
Thus, we can multiply both sides of the inequality by $\prob{X_i = t_i}$ and sum over all $t_i \in \cX$ to get:
\begin{align*}
&\qquad  \prob{(\cA(\bX),X_i )  \in \cO | \bbx_1^{i-1}}   =  \sum_{t_i \in \cX} \prob{X_i = t_i} \prob{(\cA(\bX),X_i) \in \cO | \bbx_1^{i-1}} \\
& \qquad \leq \sum_{x_i \in \cX}\sum_{t_i \in \cX} \prob{X_i=x_i} \prob{X_i = t_i}\left(  e^\eps \prob{\cA(\bX) \in \cO_{x_i} | \bbx_1^{i-1}, t_i} + \delta\right) \\
& \qquad \leq e^\eps  \sum_{x_i \in \cX}\prob{X_i=x_i} \prob{\cA(\bX) \in \cO_{x_i} | \bbx_1^{i-1}}+ \delta =  e^\eps \prob{\cA(\bX) \otimes X_i \in \cO | \bbx_1^{i-1}}+ \delta.
\end{align*}
We follow a similar argument to prove:
$$\prob{\cA(\bX) \otimes X_i   \in \cO  | \bbx_1^{i-1}}  \leq e^\eps  \prob{(\cA(\bX),X_i) \in \cO | \bbx_1^{i-1}}+ \delta.$$
\end{proof}
\fi
We now define the following set of ``good outcomes and prefixes'' for any $\hat \delta >0$:
\begin{equation}
\cE_i(\hat\delta) = \left\{ (a,\bbx_1^{i-1}):X_i \approx_{3\eps,\hat{\delta}} X_i|_{a, \bbx_1^{i-1}}\right\}
\label{eq:Ei}
\end{equation}

We use a technical lemma from \cite{KS14} (stated in the \ifnum\focs=1 full version \else appendix, \Cref{lem:conditioning}\fi), and \Cref{claim:bounded1} to derive the following result:
\begin{claim}
If $\cA$ is $(\eps,\delta)$-differentially private and $\bX \sim \cP^n$, then for each $i \in [n]$ and each prefix $\bbx_1^{i-1}\in \cX^{i-1}$ we have for $\hat\delta>0$ and $\delta' \stackrel{\text{def}}{=}
\tfrac{2\delta}{\hat\delta} + \tfrac{2\delta}{1-e^{-\eps}}$:
$$
\Prob{}{(\cA,\bX_1^{i-1}) \in \cE_i(\hat\delta) | \bbx_1^{i-1}} \geq 1-\delta'.
$$
\label{claim:Ei}
\end{claim}
\ifnum\focs=0
\begin{proof}
This follows directly from \Cref{lem:conditioning}:
\begin{align*}
\Prob{}{(\cA,\bX_1^{i-1}) \in \cE_i(\hat\delta) | \bbx_1^{i-1}} = \Prob{ }{X_i \approx_{3\eps,\hat\delta} X_i |_{\cA,\bbx_1^{i-1}} | \bbx_1^{i-1}} = \Prob{a \sim p\left( \cA|_{\bbx_1^{i-1}}\right) }{X_i \approx_{3\eps,\hat\delta} X_i |_{a,\bbx_1^{i-1}}}  \geq 1- \delta'
\end{align*}
\end{proof}
\fi

We now define the set of outcome/dataset prefix pairs for which the quantities $Z_i$ are not large:
\begin{equation}
\cF_i = \left\{(a,\bbx_1^{i}) : | Z_i(a,\bbx_1^i) | \leq 6 \eps \right\}.
\label{eq:Fi}
\end{equation}

Using another technical lemma from \cite{KS14} (which we state in \ifnum\focs=1 the full version\else \Cref{lem:prelims} in the appendix\fi), we prove:
\begin{claim}
Given $(a,\bbx_1^{i-1}) \in \cE_i(\hat\delta)$ and $\delta'' \stackrel{\text{def}}{=}
\tfrac{2\hat\delta}{1-e^{-3\eps}}$ we have:
$$
\Prob{}{(\cA,\bX_1^i) \in \cF_i | a,\bbx_1^{i-1}} \geq 1- \delta'' .
$$
\label{claim:Fi}
\end{claim}
\ifnum\focs=0
\begin{proof}
 Since $(a,\bbx_1^{i-1}) \in \cE_i(\hat\delta)$, we know that $X_i$ is $(3 \eps, \hat \delta)$-indistinguishable from $X_i|_{a,\bbx_1^{i-1}}$.  Using \Cref{lem:prelims}, we know that $X_i$ and $X_i|_{a,\bbx_1^{i-1}}$ are point-wise $\left(6 \epsilon, \delta'' \right)$-indistinguishable.  Thus, by definition of $\cF_i$ and $Z_i$, we have:
 \begin{align*}
 & \prob{(\cA,\bX_1^i) \in \cF_i | a, \bbx_1^{i-1}} = \prob{Z_i(\cA, \bX_1^i) \leq 6 \eps | a,\bbx_1^{i-1} } \\
 & \qquad = \Prob{x_i \sim p\left( X_i | _{a,\bbx_1^{i-1}} \right))}{\log\left( \tfrac{\prob{X_i = x_i | a, \bbx_1^{i-1}}}{\prob{X_i = x_i}} \right) \leq 6 \eps } \geq 1- \delta''
 \end{align*}
\end{proof}
\fi
We now define the ``good" tuples of outcomes and databases as
\begin{align}
\cG_i(\hat\delta) = \left\{(a,\bbx_1^i): (a,\bbx_1^{i-1}) \in \cE_i(\hat\delta) \quad \& \quad (a,\bbx_1^i)\in \cF_i \right\},
\label{eq:G}\\
 \cG_{\leq i}(\hat\delta)   = \left\{(a,\bbx_1^i) : (a,x_1) \in \cG_1(\hat\delta), \cdots, (a,\bbx_1^i) \in \cG_i(\hat\delta) \right\} \label{eq:Gvect}
\end{align}

\begin{claim}
If $\cA$ is $(\eps,\delta)$-differentially private and $\bX \sim \cP^n$, then
$$
\prob{(\cA,\bX_1^i) \in \cG_i(\hat\delta)} \geq 1-\delta' -  \delta''
$$
for $\delta'$ and $\delta''$ given in \Cref{claim:Ei} and \Cref{claim:Fi}, respectively.
\label{claim:Gi}
\end{claim}
\ifnum\focs=0
\begin{proof}
We have:
\begin{align*}
\prob{(\cA,\bX_1^i) \notin\cG_i(\hat\delta)}& = \prob{(\cA,\bX_1^{i-1}) \notin \cE_i(\hat\delta) \qquad \text{or} \qquad (\cA,\bX_1^i) \notin \cF_i)}\\
& =  1-\prob{(\cA,\bX_1^{i-1}) \in \cE_i(\hat\delta) \qquad \text{and} \qquad (\cA,\bX_1^i) \in \cF_i)}  \\
& = 1 - \sum_{(a,\bbx_1^{i-1}) \in \cE_i(\hat\delta)} \prob{(\cA,\bX_1^{i-1}) = (a,\bbx_1^{i-1})} \prob{(\cA,\bX_1^i) \in \cF_i | a,\bbx_1^{i-1}} \\
& \leq 1 - \sum_{(a,\bbx_1^{i-1}) \in \cE_i(\hat\delta)} \prob{(\cA,\bX_1^{i-1}) = (a,\bbx_1^{i-1})} \cdot (1-\delta'') \\
& = 1 - (1-\delta'') \prob{(\cA,\bX_1^{i-1}) \in \cE_i(\hat\delta)} \leq 1 - (1-\delta'')(1-\delta') = \delta' + \delta'' - \delta'\delta''
\end{align*}
where the last two inequalities follow from \Cref{claim:Fi} and \Cref{claim:Ei}, respectively.
\end{proof}
\fi
Having shown a high probability bound on the terms $Z_i$, our next step is to bound their expectation so that we can continue towards our goal of applying Azuma's inequality. Note a complicating factor -- throughout the argument, we need to condition on the event $(\cA, \bX_1^i) \in \cF_i$ to ensure that $Z_i$ has bounded expectation.

We will use the following shorthand notation for conditional expectation:
\begin{align*}& \Ex{}{Z_i(\cA,\bX_1^i) | a,\bbx_1^{i-1}, \cF_i} \\ & \quad 
\stackrel{def}{=} \Ex{}{Z_i(\cA,\bX_1^i) | \cA=a, \bX_1^{i-1} = \bbx_1^{i-1}, (\cA,\bX_1^i) \in \cF_i}, 
\end{align*}
with similar notation for sets $\cG_i(\hat\delta),\cG_{\leq i}(\hat\delta)$.

\begin{lemma}\label{lem:exp_Z}
Let $\cA$ be $(\eps,\delta)$-differentially private and $\bX
\sim \cP^n$.  Given $(a,\bbx_1^{i-1}) \in \cE_{i}(\hat\delta)$, for all %$0<\eps \leq 1$
$\eps \in (0,1/2]$ and %$0<\hat\delta\leq\tfrac{1-e^{-3\eps}}{2\left( 2 e^{3\eps} + 1 - e^{-3\eps}\right)}$,
$\hat{\delta} \in \left (0,\eps/15 \right ]$,%\tfrac{1-e^{-3\eps}}{2\left( 2 e^{3\eps} + 1 - e^{-3\eps}\right)} \right ]$, 
%\omnote{Modified range of $\delta$ according to parameter ranges in \Cref{thm:main}. Can revert if required.}, 
%\rynote{Modified upper bound of $\hat\delta$}\omnote{Added a statement at the end of the section regarding the updated bounds on $\hat{\delta}$}
\begin{equation*}
\Ex{}{Z_i(\cA,\bX_1^i) | a,\bbx_1^{i-1}, \cF_i} = O(\eps^2 + \hat\delta).
\end{equation*}
\ifnum\focs=0
More precisely, $\Ex{}{Z_i(\cA,\bX_1^i) | a,\bbx_1^{i-1}, \cF_i} \leq \nu (\hat{\delta})$, where $\nu (\hat{\delta})$ is defined in \eqref{eqn: nu}.
\fi
\end{lemma}

\ifnum\focs=0
\begin{proof}
Given an outcome and prefix $(a,\bbx_1^{i-1})\in\cE_i(\hat\delta)$, we define the set of data entries
$$
\cXax = \{x_i\in \cX:  (a,\bbx_1^i) \in \cF_i \}. %\{x_i: \exists\bbx_{i+1}^n \in \cX^{n-i} \text{ s.t. } (a,\bbx_1^i) \in \cF_i \}.
$$

We then have:
$$
\Ex{}{Z_i(\cA,\bX_1^i) | a,\bbx_1^{i-1}, \cF_i} = \sum_{x_i \in \cXax} \prob{X_i = x_i|a,\bbx_1^{i-1},\cF_i} \log\left( \tfrac{\prob{X_i = x_i | a,\bbx_1^{i-1}}}{\prob{X_i = x_i}} \right)
$$ Here, our goal is to mimic the proof of the ``advanced composition theorem'' of \cite{DRV10} by adding a term that looks like a KL divergence term (see \Cref{defn: KL}). In our case, however, the sum is not over the entire set $\cX$, and so it is not a KL-divergence, which leads to some additional complications.  Consider the following term:

\begin{align*}
\sum_{x_i \in \cXax} & \prob{X_i = x_i} \log\left(\tfrac{\prob{X_i = x_i | a,\bbx_1^{i-1}}}{\prob{X_i = x_i}} \right) \\
& = \prob{X_i \in \cXax} \sum_{x_i \in \cXax} \tfrac{\prob{X_i = x_i}}{\prob{X_i \in \cXax}} \log\left(\tfrac{\prob{X_i = x_i | a,\bbx_1^{i-1}}}{\prob{X_i = x_i}} \right) \\
& \leq \log\left(\tfrac{\prob{X_i \in \cXax|a,\bbx_1^{i-1}}}{\prob{X_i \in \cXax}} \right) = \log\left(\tfrac{1 - \prob{X_i \notin \cXax|a,\bbx_1^{i-1}}}{1- \prob{X_i \notin \cXax}} \right)
\end{align*}
where the inequality follows from Jensen's inequality.  Note that, because $(a,\bbx_1^{i-1}) \in \cE_i(\hat\delta)$, we have for $\hat\delta>0$:
$$
\prob{X_i \notin \cXax} \leq e^{3\eps} \prob{X_i \notin \cXax | a,\bbx_1^{i-1}} + \hat\delta.
$$
We now focus on the term $\prob{X_i \notin \cXax | a,\bbx_1^{i-1}} $.  Note that $x_i \notin\cXax \Leftrightarrow (a,\bbx_1^{i}) \notin \cF_i$.  Thus,
\begin{align*}
\prob{X_i \notin \cXax | a,\bbx_1^{i-1}} & = \prob{(\cA,\bX_1^i)\notin \cF_i | a,\bbx_1^{i-1}} \stackrel{\text{def}}{=} q
\end{align*}

Note that $q \leq \delta''$ by \Cref{claim:Fi}. Now, we can bound the following:
%We then use the inequality $(-x - 2x^2)\log(e) \leq \log(1-x) \leq -x\log(e)$ for $0< x \leq 1/2$, so that for $\hat\delta$ bounded as in the lemma statement, we can bound the following:
\begin{align*}
\sum_{x_i \in \cXax} & \prob{X_i = x_i} \log\left(\tfrac{\prob{X_i = x_i | a,\bbx_1^{i-1}}}{\prob{X_i = x_i}} \right)  \leq \log(1-q) - \log(1- (e^{3\eps}q + \hat{\delta})  ) \\
& \leq \log(e) \cdot(- q+ e^{3\eps}q + \hat{\delta} + 2(e^{3\eps}q + \hat{\delta})^2) = \log(e) \cdot ( (e^{3\eps}-1)q + \hat{\delta} + 2(e^{3\eps}q + \hat{\delta})^2 )\\
& \leq \log(e) \cdot \left ( (e^{3\eps}-1)\dfrac{2\hat{\delta}}{1-e^{-3\eps}} + \hat{\delta} + 2\hat{\delta}^2 \cdot\left(\dfrac{2e^{3\eps}}{1-e^{-3\eps}} + 1\right)^2 \right )  \\
& = \hat{\delta}(\log(e) (2e^{3\eps}+1))+  \hat{\delta}^2\left ( 2\log(e)\left( \dfrac{4e^{12\eps}+4e^{9\eps}-3e^{6\eps}-2e^{3\eps}+1}{e^{6\eps}-2e^{3\eps} + 1} \right)\right ) \stackrel{\text{def}}{=} \tau(\hat\delta)
%& =\log(e) \cdot \left ( \hat{\delta} (2e^{3\eps}+1)+ \hat{\delta}^2\left( 1 + \dfrac{2e^{6\eps}}{(e^{3\eps}-1)^2}\Bigg( 4 e^{6\eps} + 4 e^{3\eps} - 3 - 2 e^{-3\eps} + e^{-6\eps}\Bigg ) \right) \right) \stackrel{\text{def}}{=} \tau(\hat\delta)
\end{align*}
where the second inequality follows by using the inequality $(-x - 2x^2)\log(e) \leq \log(1-x) \leq -x\log(e)$ for $0< x \leq 1/2$, and as $(e^{3\eps}q + \hat{\delta}) \leq 1/2$ for $\eps$ and $\hat\delta$ bounded as in the lemma statement.

We then use this result to upper bound the expectation we wanted:
\begin{align*}
& \Ex{}{Z_i(\cA,\bX_1^i) | a,\bbx_1^{i-1}, \cF_i} \\
& \qquad \leq \sum_{x_i \in \cXax} \prob{X_i = x_i|a,\bbx_1^{i-1},\cF_i} \log\left( \tfrac{\prob{X_i = x_i | a,\bbx_1^{i-1}}}{\prob{X_i = x_i}} \right)  \\
& \qquad \qquad - \sum_{x_i \in \cXax} \prob{X_i = x_i} \log\left(\tfrac{\prob{X_i = x_i | a,\bbx_1^{i-1}}}{\prob{X_i = x_i}} \right) + \tau(\hat\delta) \\
& \qquad = \sum_{x_i \in \cXax} \left(\prob{X_i = x_i|a,\bbx_1^{i-1},\cF_i} - \prob{X_i = x_i}\right) \log\left( \tfrac{\prob{X_i = x_i | a,\bbx_1^{i-1}}}{\prob{X_i = x_i}} \right)+ \tau(\hat\delta) \\
& \qquad \leq 6\eps \sum_{x_i \in \cXax} |\prob{X_i = x_i|a,\bbx_1^{i-1},\cF_i} - \prob{X_i = x_i}|+ \tau(\hat\delta) \\
%& \qquad \leq 6\eps \sum_{x_i \in \cXax} \Bigg |\dfrac{\prob{X_i = x_i \land \cF_i|a,\bbx_1^{i-1}}}{\prob{\cF_i|a,\bbx_1^{i-1}}} - \prob{X_i = x_i} \Bigg |+ \tau(\hat\delta) \\
& \qquad \leq 6\eps \sum_{x_i \in \cXax} \prob{X_i = x_i} \max\left\{ \dfrac{e^{6\eps} }{\prob{(\cA,\bX_1^i) \in \cF_i|a,\bbx_1^{i-1}}} - 1, 1 - \dfrac{e^{-6\eps} }{\prob{(\cA,\bX_1^i) \in \cF_i|a,\bbx_1^{i-1}}}  \right\} \\
& \qquad \qquad + \tau(\hat\delta) \\ % \tag{As $x_i \in \cXax$, $\prob{X_i = x_i \land \cF_i|a,\bbx_1^{i-1}} = \prob{X_i = x_i|a,\bbx_1^{i-1}}$}\\
%& \qquad \leq 6\eps \sum_{x_i \in \cXax} \left(\tfrac{e^{6\eps}}{\prob{(\cA,\bX_1^i) \in \cF_i| a,\bbx_1^{i-1}}} - 1\right) \min\left\{ \prob{X_i = x_i | a, \bbx_1^{i-1}},\prob{X_i = x_i} \right\} + \tau(\hat\delta) \\
& \qquad \leq 6\eps \left(\tfrac{e^{6\eps}}{1- \tfrac{2\hat\delta}{1-e^{-3\eps}}} - 1\right) + \tau(\hat\delta) 
%& \qquad \leq \eps^2\left ( \dfrac{36 + 72 \hat{\delta}}{1 + 3\eps - \hat{\delta}(6 \eps + 2)} \right ) + \hat{\delta} \left ( \log e \left (2 + 6\eps \right ) + \dfrac{12 \eps}{1 + 3\eps - \hat{\delta}(6 \eps + 2)}  \right ) - \hat{\delta}^2 \left ( \log e \left ( \dfrac{8 + 288 \eps}{6\eps}\right )\right ) \\
 \leq 6\eps\left(e^{6\epsilon} \left( 1+ \tfrac{4\hat\delta}{1-e^{-3\epsilon} } \right) - 1 \right) + \tau(\hat\delta)\\
 & \qquad \leq 72 \epsilon^2 + \hat\delta\left( \tfrac{24 e^{6\epsilon} }{1-e^{-3\epsilon}}+\log(e) (2e^{3\eps}+1)  \right) +  \hat{\delta}^2\left ( 2\log(e)\left( \dfrac{4e^{12\eps}+4e^{9\eps}-3e^{6\eps}-2e^{3\eps}+1}{e^{6\eps}-2e^{3\eps} + 1} \right)\right ) \\
%%% Bad version %%%
%& \qquad \leq 6\eps^2 \left (\dfrac{ e^9 - 9 + 6 \hat{\delta}}{3\eps - 2\hat{\delta}(1+3\eps)} \right ) + \hat\delta \left ( \dfrac{12\eps}{3\eps - 2\hat\delta(1+ 3\eps)} + \log (e) (2e^3\eps + 3) \right ) \\
%& \qquad \qquad - \hat{\delta}^2 \left ( \log (e) \left ( \dfrac{\eps(4e^{12} +4 e^9 -24)+4}{\eps(e^3-3)}\right)\right) \\
& \qquad \stackrel{\text{def}}{=} \nu(\hat\delta) \numberthis \label{eqn: nu}
%%%%%%%%%%%%
\end{align*}
where the third inequality follows from the definition of $\cF_i$, the fourth inequality follows from \Cref{claim:Fi}
%and using the fact that \rynote{we do not get into $cosh$ for such a simple inequality.  I think we do not need that explanation} $cosh(y) \geq 1$ for $y \in (0,1]$, where $cosh()$ denotes the hyperbolic cosine function
and the last inequality follows by substituting the value of $\tau(\hat\delta)$, and using the inequalities $1 + y \leq e^y$ and $e^{ky} \leq 1 + e^ky$ for $y \in (0,0.5],k>1$.
\end{proof}
\fi

Finally, we need to apply Azuma's inequality \ifnum\focs=0(stated in \Cref{thm:azuma})\fi to a set of variables that are bounded with probability $1$, not just with high probability. Towards this end, we define variables $T_i$ that will match $Z_i$ for ``good events'', and will be zero otherwise---and hence, are always bounded:
\begin{equation}
T_i(a,\bbx_1^i) = \begin{cases}
             Z_i(a,\bbx_1^i)  & \text{if }  (a,\bbx_1^i)\in \cG_{\leq i}(\hat\delta) \\
             0  & \text{otherwise }
       \end{cases}
\label{eq:Ti}
\end{equation}

The next lemma verifies that the variables $T_i$ indeed satisfy the requirements of Azuma's inequality:
\begin{lemma}
Let $\cA$ be $(\eps,\delta)$-differentially private and $\bX \sim \cP^n$.  The variables $T_i$ defined in \eqref{eq:Ti} are bounded by $6 \eps$ with probability $1$, and for any $(a,\bbx_1^{i-1}) \in \cY \times \cX^{i-1}$ and $\hat\delta \in [0,\epsilon/15]$,
\begin{align}
\Ex{}{T_i(\cA,\bX_1^i) | a, \bbx_1^{i-1}}  = O(\eps^2 + \hat\delta/\epsilon), \label{eqn: ex_t_i}
\end{align}
where the bound does not depend on $n$ or $i$. 
\ifnum\focs=0
More precisely, $\Ex{}{T_i(\cA,\bX_1^i) | a, \bbx_1^{i-1}}  \leq \nu (\hat{\delta})$, where $\nu (\hat{\delta})$ is defined in \eqref{eqn: nu}. 
\fi
%for $\nu(\hat\delta)$ given in \Cref{lem:exp_Z}.
\end{lemma}
\ifnum\focs=1
We can then apply Azuma's inequality to the sum of $T_i(a, \bbx_1^i)$, where each term will match $Z_i(a,\bbx_1^i)$ for most $(a,\bbx_1^i)$ coming from $(\cA(\bX), \bX_1^i)$ for each $i \in [n]$.  Note that, from \Cref{lem:boundmaxinfo}, we know that a bound on $\sum_{i=1}^n Z_i(a,\bbx_1^i)$ with high probability will give us a bound on approximate max-information.  See the full version for a formal analysis.   
\else
\begin{proof}
%Observe that:
By definition, $T_i(\cA,\bX_1^i)$ takes values only in $[-6\eps,6\eps]$. Thus, 
\begin{align*}
\Prob{}{|T_i(\cA,\bX_1^i)| \leq 6 \eps} = 1.% \label{eqn: t_i}
\end{align*}

Now, given $(a,\bbx_1^{i-1})  \in \cE_i(\hat\delta) \cap \cG_{\leq i-1}(\hat\delta)$, we can see that:
\begin{align*}
\Ex{}{T_i(\cA,\bX_1^i) \Big | a,\bbx_1^{i-1},\cG_{\leq i}^c(\hat\delta)} = 0.
\end{align*}

Further, given $(a,\bbx_1^{i-1})\in \cE_i(\hat\delta) \cap \cG_{\leq i-1}(\hat\delta)$, we have:
\begin{align*}
& \Ex{}{T_i(\cA,\bX_1^i) | a,\bbx_1^{i-1}, \cG_{\leq i}(\hat\delta)} = \sum_{x_i : (a,\bbx_1^i) \in \cF_i} T_i (a,\bbx_1^i) \prob{X_i = x_i | a,\bbx_1^{i-1},\cG_{\leq i}(\hat\delta)}\\
& \qquad = \sum_{x_i : (a,\bbx_1^i) \in \cF_i} Z_i (a,\bbx_1^i) \prob{X_i = x_i | a,\bbx_1^{i-1}, \cG_{\leq i}(\hat\delta)}  = \sum_{x_i : (a,\bbx_1^i) \in \cF_i} Z_i (a,\bbx_1^i) \prob{X_i = x_i | a,\bbx_1^{i-1},\cF_i} \\
& \qquad = \Ex{}{Z_i(\cA,\bX_1^i) | a,\bbx_1^{i-1}, \cF_i}  = O(\eps^2 + \hat\delta/\epsilon)
\end{align*}
where the second equality follows from \eqref{eq:Ti}, %\omnote{Added second equality to make it easier to follow, can revert to the old version if required}
 and the last equality follows from \Cref{lem:exp_Z}.  For any $(a,\bbx_1^{i-1}) \notin \cE_i(\hat\delta) \cap \cG_{\leq i-1}(\hat\delta)$, we have that the conditional expectation is zero.  This proves the lemma.
\end{proof}
\fi

\ifnum\focs=0
We are now ready to prove our main theorem.
\begin{proof}[Proof of \Cref{thm:main}]
For any constant $\nu$, we have:
\begin{align*}
& \Prob{}{\sum\limits_{i=1}^n Z_i(\cA,\bX_1^i)  > n\nu + 6t\eps\sqrt{n}} \\%=  \Prob{}{\sum\limits_{i=1}^n Z_i(\cA,\bX_1^i) > n\nu + 6t\eps\sqrt{n} \wedge \bigcap\limits_{j=1}^{n} G_j} \\
%& \quad \quad +  \Prob{}{\sum\limits_{i=1}^n Z_i(\cA,\bX_1^i) > n\nu + 6t\eps\sqrt{n} \wedge \bigcup\limits_{j=1}^{n} \overline{G_j}} \\
& \leq \Prob{}{\sum\limits_{i=1}^n Z_i(\cA,\bX_1^i) > n\nu + 6t\eps\sqrt{n} \cap (\cA,\bX) \in \cG_{\leq n}(\hat\delta)}+ \Prob{}{(\cA,\bX) \notin \cG_{\leq n}(\hat\delta)} \\
& =  \Prob{}{\sum\limits_{i=1}^n T_i(\cA,\bX_1^i) > n\nu + 6t\eps\sqrt{n} \cap (\cA,\bX) \in  \cG_{\leq n}(\hat\delta) } + \Prob{}{(\cA,\bX) \notin  \cG_{\leq n}(\hat\delta)}
\end{align*}
We then substitute $\nu$ by $\nu(\hat\delta)$ as defined in \Cref{eqn: nu}, and apply a union bound on $\prob{(\cA,\bX) \notin  \cG_{\leq n}(\hat\delta)}$ using \Cref{claim:Gi} to get
\begin{align*}
 \Prob{}{\sum\limits_{i=1}^n Z_i(\cA,\bX_1^i)  > n\nu(\hat\delta) + 6t\eps\sqrt{n}} & \leq \Prob{}{\sum\limits_{i=1}^n T_i(\cA,\bX_1^i) > n\nu(\hat\delta) + 6t\eps\sqrt{n}  } + n(\delta' + \delta'') \\
&   \leq e^{-t^2/2} + n(\delta' + \delta'')
\end{align*}
where the two inequalities follow from \Cref{claim:Gi} and \Cref{thm:azuma}, respectively.  Therefore,
\begin{align*}
\Prob{ }{Z(\cA(\bX),\bX) > n\nu(\hat\delta) + 6t\eps\sqrt{n}} \leq e^{-t^2/2} + n(\delta' + \delta'')  \stackrel{def}{=} \beta(t,\hat\delta)
\end{align*}

From \Cref{lem:boundmaxinfo}, we have
$I^{\beta(t,\hat\delta)}_\infty(\bX;\cA(\bX)) \leq n\nu(\hat\delta)
+ 6t\eps\sqrt{n}.$  We set the parameters $t = \eps \sqrt{2n}$ and $\hat\delta = \sqrt{\eps\delta}/15$ to obtain our result. Note that setting $\hat{\delta}= \sqrt{\eps\delta}/15$ does not violate the bounds on it stated in the statement of \Cref{lem:exp_Z}.  %\omnote{Added the last statement to state that $\hat{\delta}$ does not exceed the upper bound set in \Cref{lem:exp_Z}}

%\begin{align*}\label{eq:maxinfobound}
%I^{\beta(t,\hat\delta)}_\infty(\bX;\cA(\bX)) & \leq n\nu(\hat\delta) + 6t\eps\sqrt{n} \\
%& = n \Bigg ( 6  \eps (e^{6\eps} -1) + (2e^{3\eps}+1)\hat{\delta} + \hat{\delta}^2 + \dfrac{4e^{6\eps}\hat{\delta}^2}{e^{3\eps}-1}\Bigg( \dfrac{e^{6\eps} + e^{3\eps} - 1}{e^{3\eps}-1} \Bigg ) \Bigg ) + 6t\eps\sqrt{n}
%\end{align*}
\end{proof}
\fi
%%% Local Variables:
%%% mode: latex
%%% TeX-master: "main"
%%% End:

%\input{prelims_linear_codes}
%\input{application}

\section{A Counterexample to Nontrivial Composition and a Lower Bound for Non-Product Distributions}
\label{sec:lowerbound}
\RestyleAlgo{boxed}

It is known that algorithms with bounded description length have bounded approximate max-information \cite{DFHPRR15NIPS}.  In section \ref{sec: max_info}, we showed that $(\eps, \delta)$-differentially private algorithms have bounded approximate max-information when the dataset is drawn from a product distribution. In this section, we show that although approximate max-information composes adaptively \cite{DFHPRR15NIPS}, one cannot always run a bounded description length algorithm, followed by a differentially private algorithm, and expect the resulting composition to have strong generalization guarantees. In particular, this implies that $(\eps,\delta)$-differentially private algorithms cannot have any nontrivial bounded max-information guarantee over non-product distributions.

Specifically, we give an example of a pair of algorithms $\mathcal{A}$ and $\mathcal{B}$ such that $\mathcal{A}$ has output description length $o(n)$ for inputs of length $n$, and $\mathcal{B}$ is $(\epsilon,\delta)$-differentially private, but the adaptive composition of $\mathcal{A}$ followed by $\mathcal{B}$ can be used to exactly reconstruct  the input database with high probability. In particular, it is easy to overfit to the input $\bX$ given $\mathcal{B}(\bX; \mathcal{A}(\bX))$, and hence, no nontrivial generalization guarantees are possible. Note that this does not contradict our results on the max-information of differentially private algorithms for \emph{product distributions}: even if the database used as input to $\cA$ is drawn from a product distribution, the distribution on the database is no longer a product distribution \emph{once conditioned on the output of $\cA$}. The distribution of $\cB$'s input violates the hypothesis that is used to prove a bound on the max-information of $\cB$.
%Note that, if we set $\eps = n^{-1/3}$ and $\delta = 2^{-\sqrt[3]{n}}$, and if $\cB$'s input was from a product distribution, then from \Cref{thm:main}, we have that the approximate max-information of $\cB$ would be $o(n)$ as well.

%In other words, both of these algorithms can individually be used adaptively in sequences that guarantee generalization, but their composition with one another in a particular order can eliminate generalization.

\begin{comment}
We will use linear codes in our example. So before we present the example, let us look at some preliminaries for linear codes. For a linear code $C\subseteq \{ 0,1 \}^n$, let us denote its minimum distance by $t$ and its rank by $k$. Thus, there are $2^k$ codewords in $C$. Also define the rate $R$ of $C$ as $R = \dfrac{k}{n}$, and the relative distance $\delta$ of $C$ as $\delta = \dfrac{t}{n}$. Now, let us look at the Gilbert-Varshamov bound:
\begin{theorem}
There always exists a linear code $C$ with rate $R$ and relative distance $\delta$ if
\begin{enumerate}
\item $\delta < \dfrac{1}{2}$, and
\item $R \leq 1 - H(\delta)$, where $H(x)$ is the binary entropy function defined as $H(x) = - x log_2(x) - (1-x) log_2(1-x)$.
\end{enumerate}
\end{theorem}
\end{comment}

\begin{theorem}
Let $\cX = \{ 0,1 \}$ and $\cY =  \{\cX^n \cup \{ \bot \}\}$. Let $\bX$ be a uniformly distributed random variable over $\cX^n$. For $n > 64e$, for every $\eps \in \left ( 0,\frac{1}{2} \right ], \delta \in \left ( 0,\frac{1}{4} \right ]$,  there exists an integer $r>0$ and randomized algorithms $\mathcal{A}: \cX^n \rightarrow \{0,1\}^r$, and $\mathcal{B}: \cX^n \times \{0,1\}^r \rightarrow \cY$, such that:
\begin{enumerate}
\item $ r = O \Bigg (\dfrac{\log (1/\delta) \log n}{\eps}\Bigg )$ and $ I^\beta_{\infty}(\bX;\mathcal{A}(\bX)) \leq r + \log (\frac{1}{\beta})$ for all $\beta > 0$;
%. Since $r$ is the description length of $\mathcal{A}$, $ I^\beta_\infty(\bX;\mathcal{A}(\bX)) \leq r + \log (\frac{1}{\beta})$ for all $\beta > 0$.
\item for every $\bba \in \{ 0, 1 \}^r$, $\mathcal{B}(\bX,\bba)$ is $(\eps, \delta)$-differentially private and $ I^\beta_{\infty}(\bX;\mathcal{B}(\bX, \bba)) \leq 1$ for all $\beta \geq 2\delta$;
%. Furthermore , $ I^\beta_\infty(\bX;\mathcal{B}(\bX, \bba)) \leq 1$ for all $\beta \geq 2\delta$.
\item for every $\bbx \in \cX^n$, with probability at least $1 - \delta$, we have that $\mathcal{B}(\bbx; \mathcal{A}(\bbx))) = \bbx$. In particular, $I^\beta_{\infty}(\bX, \mathcal{B}(\bX; \mathcal{A}(\bX)))\geq n - 1$ for all $0 < \beta \leq \frac{1}{2} - \delta$.
%Consequently, $I^\beta_\infty(\bX, \mathcal{B}(\bX; \mathcal{A}(\bX))) \geq n - 1$ for all $0 < \beta \leq \frac{1}{2} - \delta$.
\end{enumerate}
\label{thm: naive}
\end{theorem}

De \cite{De12} showed that the \emph{mutual information} of $(\eps,\delta)$-differentially private protocols can be large: if $\frac{1}{\eps}\log \left ( \frac{1}{\delta} \right ) = O(n)$, then there exists an $(\eps,\delta)$-differentially private algorithm $\cB$ and a distribution $\cS$ such that for $\bX \sim \cS$, $I(\bX;\cB(\bX)) = \Omega(n)$, where $I$ denotes mutual information. De's construction also has large approximate max-information.
%from \cite{De12} also implies a version of Corollary \ref{cor:nonproduct}.

By the composition theorem for approximate max-information (given in the \ifnum\focs=1 full version\else appendix in Lemma \ref{lem:maxinfocomp}\fi), our construction implies a similar bound:
\begin{corollary}
\label{cor:nonproduct}
There exists an $(\epsilon,\delta)$-differentially private mechanism $\cC: \cX^n \to \cY$ such that $I_\infty^{\beta_2}(\cC,n) \geq n - 1 - r - \log(1/\beta_1)$ for all $\beta_1 \in (0, 1/2 - \delta)$ and $\beta_2 \in (0, 1/2 - \delta - \beta_1)$, where $r = O \left( \frac{\log(1/\delta)\log(n) }{\epsilon} \right)$.
\end{corollary}

We adapt ideas from De's construction in order to prove Theorem \ref{thm: naive}. In De's construction, the input is not drawn from a product distribution---instead, the support of the input distribution is an error-correcting code, meaning that all points in the support are far from each other in Hamming distance. For such a distribution, De showed that adding the level of noise required for differential privacy does not add enough distortion to prevent decoding of the dataset.

%The intuition behind using codes is that to counteract the (mostly low) amount of noise added by a differentially private protocol, we want a distribution on the inputs such that no two elements in its support are `close' to each other.

Our construction adapts De's idea. Given as input a \emph{uniformly random} dataset $\bbx$, we show a mechanism $\cA$ which outputs a short description of a code that contains $\bbx$. Because this description is short, $\cA$ has small max-information. The mechanism $\cB$ is then parameterized by this short description of a code. Given the description of a code and the dataset $\bbx$, $\cB$ approximates (privately) the distance from $\bbx$ to the nearest \emph{codeword}, and outputs that codeword when the distance is small. When $\cB$ is composed with $\cA$, we show that it outputs the dataset $\bbx$ with high probability.
%and attempts to decode the data set $\bbx$ after a perturbation that is sufficient to guarantee differential privacy. We show that the decoding succeeds with high probability.

\ifnum\focs=1
We define the mechanisms $\cA$ and $\cB$ from the theorem statement in \Cref{alg:A} and \Cref{alg:naive}, respectively.

\textit{Brief description of $\cA$:} For any input $\bbx \in \cX^n$, mechanism $\cA$ returns a vector $\bba_\bbx \in \{ 0, 1 \}^r$ such that $\bbx \in C_{\bba_\bbx}$, where $C_{\bba_\bbx} = \{ \bbc \in \cX^n : H\bbc = \bba_\bbx \}$ is an affine code with minimum distance $t$.  We give further details in the full version of the paper.

\begin{algorithm}
\caption{$\mathcal{A}$\label{alg:A}}
\KwIn{$\bbx \in \{ 0, 1 \}^n$}
\KwOut{$\bba_\bbx \in \{ 0, 1 \}^r $}

Return $H\bbx$ (multiplication in $\mathbb{F}_2$).
\end{algorithm}
%Therefore, the description length of the output of $\mathcal{A} = r = O(t \log n)$. As $t$ is constant, we can set appropriate constants in $r$ so that the description length of $\mathcal{A}$ is at most $\dfrac{\ln (1/\delta) \ln n}{\eps}$.
%Now, for any $\bbx \in \cX^n$, let $\bba_\bbx = \mathcal{A}(\bbx) $. Consider the code $C_{\bba_\bbx} = \{ \bbc \in \cX^n : H\bbc = \bba_\bbx \}$. From \Cref{lem:codes}, $C_{\bba_\bbx}$ is an affine code with minimum distance $t$. We also have $\bbx \in C_{\bba_\bbx}$, since $H\bbx = \bba_\bbx$, again from \Cref{lem:codes}.
%Define $w = \left ( \dfrac{t-1}{4} - \dfrac{\log (1/\delta)}{\eps}\right )$. Observe that $w \geq \dfrac{\log (1/\delta)}{\eps}$.

%For each $\bba$, when $d_\bbx < \dfrac{t-1}{2}$, we can say that $\bbx^* = \arg\min\limits_{\bby \in C_\bba}( dist_{Hamm}(\bbx,\bby))$ is the same for all neighbors of $\bbx$ as well as for $\bbx$ (as $\bbx^*$ does depend on $\bba$).

\textit{Brief description of $\mathcal{B}_{\eps,\delta}$:} For any input $\bbx \in \cX^n$ and $\bba \in \{ 0, 1 \}^r$, mechanism $\mathcal{B}_{\eps,\delta}$ first computes $d_\bbx$, which is the distance of $\bbx$ from $f(\bbx)$, i.e.,  the nearest codeword to $\bbx$ in code $C_\bba$. Next, it sets $\hat{d}_\bbx$ to be $d_\bbx$ perturbed with Laplace noise $L \sim \text{Lap}(1/\eps)$. It returns $f(\bbx)$ if $\hat{d}_\bbx$ is below a threshold $w \stackrel{def}{=} \left ( \dfrac{t-1}{4} - \dfrac{\log (1/\delta)}{\eps}\right )$, and $\bot$ otherwise.

\begin{algorithm}
\caption{$\mathcal{B}_{\eps,\delta}$\label{alg:naive}}
\KwIn{$\bbx \in \{ 0, 1 \}^n$ (private) and $\bba \in \{ 0, 1 \}^r$(public)}
\KwOut{$\bbb \in \cY $}
Compute the distance of $\bbx$ to the nearest codeword in code $C_\bba$. Let $d_\bbx =\min\limits_{\bbc \in C_\bba}( dist_{Hamm}(\bbx,\bbc))$ and $f(\bbx) =  \arg\min\limits_{\bbc \in C_\bba}( dist_{Hamm}(\bbx,\bbc))$ (breaking ties arbitrarily).

Let $\hat{d}_\bbx = d_\bbx + L$, where $L \sim \text{Lap}(1/\eps)$, and $\text{Lap}(c)$ denotes a random variable having Laplace(0,$c$) distribution.

\eIf{$\hat{d}_\bbx < \left ( \dfrac{t-1}{4} - \dfrac{\log (1/\delta)}{\eps}\right )$}{Return $f(\bbx)$.}{Return $\bot$.}
\end{algorithm}

We prove \Cref{thm: naive} in the full version.
\else
Before getting into the details of our result, we present some preliminaries for linear codes.

%This can be achieved easily by using linear codes. Hence, we present a brief overview of linear codes in \Cref{sec: prel_codes} which will be useful in our proof of \Cref{thm: naive}.

\subsection{Preliminaries for Linear Codes} \label{sec: prel_codes}

For the current and the next subsections, we limit our scope to $\mathbb{F}_2$, i.e., the finite field with 2 elements. First, we define linear codes: %as we will use them extensively in the proof of the theorem in the next section.

\begin{defn}[Linear Code]
A code $C \subseteq \{ 0,1 \}^n$ of length $n$ and rank $k$ is called linear  iff it is a $k$ dimensional linear subspace of the vector space $ \mathbb{F}^n_2$. The vectors in $C$ are called codewords.
\end{defn}
The minimum distance $t$ of a linear code $C$ is $t = \min\limits_{c_1, c_2 \in C} dist_{Hamm}(c_1,c_2),$ where, $dist_{Hamm}(p,q)$ denotes the Hamming distance between binary vectors $p$ and $q$.

We now define parity check matrices, which can be used to construct linear codes. Every linear code has a parity-check matrix corresponding to it. Thus, given a parity-check matrix, one can reconstruct the corresponding linear code.
\begin{defn}[Parity-check matrix]
For a linear code $C \subseteq \{ 0,1 \}^n$ of length $n$ and rank $k$, $H\in \{ 0,1 \}^{(n-k)\times n} $ is a parity-check matrix of $C$ iff $H$ is a matrix whose null space is $C$, i.e., $c \in C$ iff $Hc = \textbf{0}$, where $\textbf{0}$ represents the zero vector.
%The minimum distance of the code is the minimum number $t$ such that every $t$ columns of $H$ are linearly independent while there exist $t+1$ columns of $H$ that are linearly dependent.
\end{defn}

%\omnote{the following theorem proves that there exists a linear code for the parameter values we have chosen, and the theorem follows from 1) the Gilbert-Varshamov bound (cited from book of Lint), and 2) $h_2(x) \leq cx\ln(\frac{1}{x})$ for $x \in (0,\frac{1}{2})$ and $ c \geq 3$. Can show proof if required.}
Now, we state a theorem which shows the existence of high-rank linear codes when the minimum distance is less than half the code length:

\begin{theorem}[From Theorem 5.1.8 in \cite{L99}]
For every $t \in (0,\frac{n}{2})$, there exists a linear code of rank $k$ such that $k \geq n - 3t\log (n)$.
\label{thm: existence}
\end{theorem}

\begin{comment}
A linear code $C\subseteq \{ 0,1 \}^n$ with minimum distance $t$ and rank $k$ has exactly $2^k$ codewords. Also let $R = \dfrac{k}{n}$ denote the rate of the code, and $C$ as $\delta = \dfrac{t}{n}$ denote its relative distance. We know that:

\begin{theorem}[Gilbert-Varshamov bound]
There exists a linear code $C\subseteq \{ 0,1 \}^n$ with rate $R$ and relative distance $\delta$ if
\begin{enumerate}
\item $\delta < \dfrac{1}{2}$, and
\item $R \leq 1 - h_2(\delta)$, where $h_2(x)$ is the binary entropy function \rynote{Can we use another variable, because $H$ is already the  parity check matrix.}\omnote{Good catch! Does $h_2$ seem good?}\omnote{There is probably another way we can show the existence of codes for our chosen parameters (without using $R, \delta$ and $h_2$!). Will verify and add soon.}\rynote{I guess $H$ is more commonly used for entropy.  Can we change the matrix variable name?}\rynote{Actually, why even have this notation here?  Do we ever need the entropy function again?} defined as $h_2(x) = - x log_2(x) - (1-x) log_2(1-x)$.
\end{enumerate}
\end{theorem}
\end{comment}

Next, we will define an affine code, which is a translation of a linear code by a fixed vector in the vector space of the linear code:

\begin{defn}[Affine Code]
Let $C \subseteq \{ 0,1 \}^n$ be a linear code of length $n$, rank $k$ and minimum distance $t$. For any vector $b \in  \{0,1\}^n$, the code defined by $C_a = \{ c + b : c \in C \}$, where $a = Hb$, is called an affine code.
\end{defn}

\begin{lemma}
If $C$ is a linear code with parity check matrix $H$ and minimum distance $t$, then the affine code $C_a$ also has minimum distance $t$.  Further, for all $c' \in C_a$, we have $H c' = a$.
\label{lem:codes}
\end{lemma}
\begin{proof}
Let $c' \in C_a$.  We know that there exists a $c \in C$ such that
$$ Hc' = H(c+b) = \textbf{0} + Hb = a.$$
\end{proof}

Lastly,  we define the concept of a Hamming ball around a point, which is helpful in understanding the point's neighborhood -- i.e., the points close to it with respect to Hamming distance.

\begin{defn}[Hamming ball]
A Hamming ball of radius $r$ around a point $p \in \{0,1\}^n$, denoted by $B_r(p)$, is the set of strings $x \in \{0,1\}^n$ such that $dist_{Hamm}(x,p) \leq r$.
\end{defn}

The volume of a Hamming ball, denoted by $Vol(B_r)$,  is independent of the point around which the ball is centered, i.e., for any point $p \in \{0,1\}^n$:
\begin{equation}
Vol(B_r) = \big |B_r(p) \big | = \sum\limits_{i=0}^{r} \big | \{ x \in \{ 0,1 \}^n : dist_{Hamm}(x,p) = i \} \big | =  \sum\limits_{i=0}^{r} {n \choose i}.
\label{eq:vol}
\end{equation}

\subsection{Proof of \Cref{thm: naive}} \label{sec: proof}
%First, let us define appropriate $\mathcal{A}$ and $\mathcal{B}$ that will help us prove the conjecture.
In this section, we  define the mechanisms $\cA$ and $\cB$ from the theorem statement, and then prove our result in three parts: First, we show that the first bullet in the theorem statement directly follows from setting the parameters appropriately and from \cite{DFHPRR15NIPS}. Next, we show the proof of the second bullet in two pieces. We start by showing that the algorithm $\cB$ that we define is differentially private, and then, we show that the approximate max-information of $\cB$ is small when its inputs are chosen independently. Lastly, we prove the third bullet by first showing that the adaptive composition of $\cA$ followed by $\cB$ results in the reconstruction of the input with high probability. Subsequently, we show that such a composition has large approximate max-information.

Before we define the mechanisms $\cA$ and $\cB$, we must set up some notation.  We fix $t$ such that $t =  \dfrac{8 \log (1/\delta)}{\eps} + 1$. We know that $t \geq 33$ because $\epsilon \in (0,1/2]$ and $\delta \in (0,1/4]$. Now, fix an ($(n-k) \times n)$ parity-check matrix $H$ for a linear code $C \subseteq \{ 0, 1 \}^n$ of rank $k$ over $\mathbb{F}_2$ where $t$ is the minimum distance of $C$ and $k = n - 3t \log n$, and let $r = n - k = 3t \log n$.  We can ensure the existence of $C$ from Theorem \ref{thm: existence}.

We define the mechanisms $\cA$ and $\cB$ from the theorem statement in \Cref{alg:A} and \Cref{alg:naive}, respectively.

\textit{Brief description of $\cA$:} For any input $\bbx \in \cX^n$, mechanism $\cA$ returns a vector $\bba_\bbx \in \{ 0, 1 \}^r$ such that $\bbx \in C_{\bba_\bbx}$, where $C_{\bba_\bbx}$ is an affine code with minimum distance $t$. This follows as $\bba_\bbx = \mathcal{A}(\bbx) = H\bbx$, and from \Cref{lem:codes}, as $C_{\bba_\bbx} = \{ \bbc \in \cX^n : H\bbc = \bba_\bbx \}$.

\begin{algorithm}
\caption{$\mathcal{A}$\label{alg:A}}
\KwIn{$\bbx \in \{ 0, 1 \}^n$}
\KwOut{$\bba_\bbx \in \{ 0, 1 \}^r $}

Return $H\bbx$ (multiplication in $\mathbb{F}_2$).
\end{algorithm}
%Therefore, the description length of the output of $\mathcal{A} = r = O(t \log n)$. As $t$ is constant, we can set appropriate constants in $r$ so that the description length of $\mathcal{A}$ is at most $\dfrac{\ln (1/\delta) \ln n}{\eps}$.
%Now, for any $\bbx \in \cX^n$, let $\bba_\bbx = \mathcal{A}(\bbx) $. Consider the code $C_{\bba_\bbx} = \{ \bbc \in \cX^n : H\bbc = \bba_\bbx \}$. From \Cref{lem:codes}, $C_{\bba_\bbx}$ is an affine code with minimum distance $t$. We also have $\bbx \in C_{\bba_\bbx}$, since $H\bbx = \bba_\bbx$, again from \Cref{lem:codes}.
%Define $w = \left ( \dfrac{t-1}{4} - \dfrac{\log (1/\delta)}{\eps}\right )$. Observe that $w \geq \dfrac{\log (1/\delta)}{\eps}$.

%For each $\bba$, when $d_\bbx < \dfrac{t-1}{2}$, we can say that $\bbx^* = \arg\min\limits_{\bby \in C_\bba}( dist_{Hamm}(\bbx,\bby))$ is the same for all neighbors of $\bbx$ as well as for $\bbx$ (as $\bbx^*$ does depend on $\bba$).

\textit{Brief description of $\mathcal{B}_{\eps,\delta}$:} For any input $\bbx \in \cX^n$ and $\bba \in \{ 0, 1 \}^r$, mechanism $\mathcal{B}_{\eps,\delta}$ first computes $d_\bbx$, which is the distance of $\bbx$ from $f(\bbx)$, i.e.,  the nearest codeword to $\bbx$ in code $C_\bba$. Next, it sets $\hat{d}_\bbx$ to be $d_\bbx$ perturbed with Laplace noise $L \sim \text{Lap}(1/\eps)$. It returns $f(\bbx)$ if $\hat{d}_\bbx$ is below a threshold $w \stackrel{def}{=} \left ( \dfrac{t-1}{4} - \dfrac{\log (1/\delta)}{\eps}\right )$, and $\bot$ otherwise.

\begin{algorithm}
\caption{$\mathcal{B}_{\eps,\delta}$\label{alg:naive}}
\KwIn{$\bbx \in \{ 0, 1 \}^n$ (private) and $\bba \in \{ 0, 1 \}^r$(public)}
\KwOut{$\bbb \in \cY $}
Compute the distance of $\bbx$ to the nearest codeword in code $C_\bba$. Let $d_\bbx =\min\limits_{\bbc \in C_\bba}( dist_{Hamm}(\bbx,\bbc))$ and $f(\bbx) =  \arg\min\limits_{\bbc \in C_\bba}( dist_{Hamm}(\bbx,\bbc))$ (breaking ties arbitrarily).

Let $\hat{d}_\bbx = d_\bbx + L$, where $L \sim \text{Lap}(1/\eps)$, and $\text{Lap}(c)$ denotes a random variable having Laplace(0,$c$) distribution.

\eIf{$\hat{d}_\bbx < \left ( \dfrac{t-1}{4} - \dfrac{\log (1/\delta)}{\eps}\right )$}{Return $f(\bbx)$.}{Return $\bot$.}
\end{algorithm}

Now, we present the proof of our theorem.

\begin{proof}[Proof of \Cref{thm: naive}, part 1.]
Observe that $r = O \Bigg (\dfrac{\log (1/\delta) \log n}{\eps}\Bigg )$ from the value assigned to $t$. We know that the second statement holds by the max-information bound for mechanisms with bounded description length from \cite{DFHPRR15NIPS}.
\end{proof}

\begin{proof}[Proof of \Cref{thm: naive}, part 2.]
First, we show that $\mathcal{B}_{\eps,\delta}$ is indeed differentially private.

\begin{lemma}
$\mathcal{B}_{\eps,\delta}(\cdot, \bba)$ is $(\eps, \delta)$-differentially private for every $\bba \in \{0,1 \}^r$.
\end{lemma}

\begin{proof}
We will prove this lemma by following the proof of Proposition 3 in \cite{ST13}.  Fix any $\bba \in \{0,1 \}^r$. Firstly, observe that for every $\bbx \in \{0,1\}^n$, there are only 2 possible outputs for $\mathcal{B}_{\eps,\delta}(\bbx, \bba)$: $\bot$ or $f(\bbx) = \arg\min\limits_{\bbc \in C_\bba}( dist_{Hamm}(\bbx,\bbc))$. Also, $\mathcal{B}_{\eps,\delta}(\bbx, \bba) = f(\bbx)$ iff $\hat{d}_\bbx = d_\bbx + L < w$ in \Cref{alg:naive}, where $d_\bbx = \min\limits_{\bbc \in C_\bba}( dist_{Hamm}(\bbx,\bbc))$ and $L \sim \text{Lap}(1/\eps)$.

Now, for any pair of points $\bbx$ and $\bbx'$ such that $dist_{Hamm}(\bbx,\bbx') = 1$, there are two possible cases:
\begin{enumerate}
\item $f(\bbx)  \neq f(\bbx')$:

In this case,
\begin{align*}
p & \stackrel{\text{def}}{=} \Prob{}{\cB_{\eps,\delta}(\bbx, \bba) = f(\bbx)} = \Prob{}{\hat{d}_\bbx < w} =  \Prob{}{d_\bbx + L < \dfrac{t-1}{4} - \dfrac{\log (1/\delta)}{\eps}}\\
& \leq \Prob{}{\dfrac{t-1}{2} + L < \dfrac{t-1}{4} - \dfrac{\log (1/\delta)}{\eps}}  \leq \Prob{}{L < - \dfrac{\log (1/\delta)}{\eps}} \leq \delta
\end{align*}
where the first inequality follows as $f(\bbx)  \neq f(\bbx')$ implies $d_\bbx > \dfrac{t-1}{2}$, and the last inequality follows from the tail property of the Laplace distribution.  Therefore, $\Prob{}{\cB_{\eps,\delta}(\bbx, \bba) = \bot} = 1-p$.

Similarly, $p' \stackrel{\text{def}}{=}\Prob{}{\cB_{\eps,\delta}(\bbx', \bba) = f(\bbx')} \leq \delta$, and consequently, $\Prob{}{\cB_{\eps,\delta}(\bbx', \bba) = \bot} = 1-p'$.

Thus, for any set $\cO \subseteq \cY$, we can bound the following difference in terms of the total variation distance $TV(\cB_{\eps,\delta}(\bbx,\bba),\cB_{\eps,\delta}(\bbx',\bba))$ (defined in the appendix)
\begin{align*}
\big |\Prob{}{\cB_{\eps,\delta}(\bbx, \bba) \in \cO} - \Prob{}{\cB_{\eps,\delta}(\bbx', \bba) \in \cO} \big | & \leq TV(\cB_{\eps,\delta}(\bbx, \bba), \cB_{\eps,\delta}(\bbx', \bba))\\
& = \dfrac{(p-0) + (p'-0) + |(1-p) - (1-p')|}{2}  \\
& = \dfrac{p + p' + |p'-p|}{2} = \max \{ p,p' \} \leq \delta
\end{align*}
\item $f(\bbx)  = f(\bbx')$:

Observe that for every $\bbx'' \in \{0,1\}^n$,  the value of $d_{\bbx''}$ can change by at most 1 if exactly one coordinate is changed in $\bbx''$. Computing $\hat{d}_{\bbx''}$ is then just an instantiation of the Laplace mechanism, given in the appendix (\Cref{thm:lap}).  Therefore, $\hat{d}_{\bbx''}$ satisfies $(\eps,0)$-differential privacy. Notice that determining whether to output $f(\bbx) = f(\bbx')$ or $\bot$ is a post-processing function of the $(\eps,0)$-differentially private $\hat d_\bbx$, and thus, by \Cref{thm: post_proc}, $\cB_{\eps,\delta}(\cdot, \bba)$ is $(\eps,0)$-differentially private for such inputs.

\begin{comment}
Now, notice that the output of $\cB_{\eps,\delta}(\bbx, \bba)$ can be determined by the value of the pair $(\hat{d}_\bbx, f_\bbx)$ in the algorithm. For any set $\cO \subseteq \cY$, define the set $\cD_\cO$ as \rynote{Already used $t$ as something else.  I think this is confusing.  Can we just say that $d_\bbx$ is a 1-sensitive function of $\bbx$, so $\hat d_{\bbx}$ is $\epsilon$-differentially private.  Then determining whether to output $f(\bbx) = f(\bbx')$ or $\bot$ is a post processing function of the differentially private $\hat d_\bbx$ (I will add the post processing lemma to the appendix).  }\omnote{Yes, we use $t$ for denoting the minimum distance in linear codes.}
$$
\cD_\cO = \{ t : \text{Setting } \hat{d}_\bbx = t \text{ in } \cB_{\eps,\delta}(\bbx, \bba) \Rightarrow \cB_{\eps,\delta}(\bbx, \bba) \in \cO\}.
$$
Thus,
\begin{align*}
\Prob{}{\cB_{\eps,\delta}(\bbx, \bba) \in \cO} = \Prob{}{\hat{d}_\bbx \in \cD_\cO} \leq e^\eps \Prob{}{\hat{d}_{\bbx'} \in \cD_\cO} = e^\eps\Prob{}{\cB_{\eps,\delta}(\bbx', \bba) \in \cO}
\end{align*}
where the inequality follows as $\hat{d}_\bbx$ is $\eps$-differentially private.
\end{comment}
\end{enumerate}
Therefore, from the above two cases, for any set $\cO \subseteq \cY$, we have that:
$$ \Prob{}{\cB_{\eps,\delta}(\bbx, \bba) \in \cO} \leq e^\eps\Prob{}{\cB_{\eps,\delta}(\bbx', \bba) \in \cO} + \delta.$$
Thus, we can conclude that $\mathcal{B}_{\eps,\delta}(\cdot, \bba)$ is $(\eps, \delta)$-differentially private for every $\bba \in \{0,1 \}^r$.
\end{proof}

 %From Proposition \ref{prop: alg_eps_delta} \omnote{Seems like we are repeating the whole algo in the prelims. I will add a short para here stating the relation to ST13.} $\mathcal{B}_{\eps,\delta}$ is $(\eps, \delta)$-DP.

Next, we look at the outcome of $\mathcal{B}_{\eps,\delta}(\bX, \bba)$ when $\bX$ is drawn uniformly over $ \cX^n$ and $\bba$ is a fixed $r$-bit string. Note that $\mathcal{B}_{\eps,\delta}(\bX, \bba)$ outputs either $\bot$ or a codeword of $C_\bba$.  Thus,
%Moreover, the balls of radius $\left ( \frac{t-1}{2} \right )$ around the codewords of $C_\bba$ are disjoint. Thus,
%\begin{align}
%\Prob{}{\mathcal{B}_{\eps,\delta}(\bX, \bba) \neq \bot} & = \Prob{}{\mathcal{B}_{\eps,\delta}(\bX, \bba) \in C_\bba} \label{eqn: not_bot}
%\end{align}
%for every $\bby \in C_\bba$, $\bby =  \arg\min\limits_{\bbc \in C_\bba}( dist_{Hamm}(\bbx,\bbc))$ for all strings $\bbx \in \cX^n$ such that $dist_{Hamm} (\bbx,\bby) \leq \Bigg ( \dfrac{t-1}{2} \Bigg )$.

%Now, we know that the probability of $\mathcal{B}_{\eps,\delta}(\bX, \bba)$ not outputting $\bot$ is exactly equal to the probability of $\hat{d_\bX}$ being less than $w$. This probability is at most the probability of $\hat{d_\bX}$ being less than $w$ when we restrict $d_\bX$ to $\left [ 0 , \frac{t-1}{2} \right ]$, because the probability of $\hat{d_\bX}$ being less than $w$ increases as we decrease the value of $d_\bX$. \rynote{The past two sentences are confusing and may not add any understanding to the proof.  Can we write what it is saying in math.}

%\rynote{What am I doing wrong below}

\begin{align*}
\Prob{}{\mathcal{B}_{\eps,\delta}(\bX, \bba) \neq \bot}
%\sum_{\bbx \in \cX^n} \prob{\cB_{\eps,\delta}(\bX,\bba) \neq \bot | \bX = \bbx}\prob{\bX = \bbx}\\
%& = \dfrac{1}{2^n} \Bigg (\sum_{\bbx \in G} \prob{\cB_{\eps,\delta}(\bX,\bba) \neq \bot | \bX = \bbx} + \sum_{\bbx \in \overline{G}} \prob{\cB_{\eps,\delta}(\bX,\bba) \neq \bot | \bX = \bbx} \Bigg ) \\
%& \leq \prob{\cB_{\eps,\delta}(\bX,\bba) \neq \bot | \bX \in G} \\
  = \prob{\hat{d}_\bX < w }
%= \Prob{}{\hat{d}_\bX < w \Bigg | \bX \in G}\prob{\bX\in G} + \Prob{}{\hat{d}_\bX < w \Bigg | \bX \notin G}\prob{\bX\notin G}\\
%&  \stackrel{????}{\leq} \Prob{}{\hat{d}_\bX < w \Bigg | \bX \in G} \\
%& = \Prob{}{\hat{d}_\bX < \Bigg ( \dfrac{t-1}{4} - \dfrac{\log (1/\delta)}{\eps}\Bigg ) \Bigg | \bX \in G } \tag{As $w = \Bigg ( \dfrac{t-1}{4} - \dfrac{\log (1/\delta)}{\eps}\Bigg )$}\\
%& = \Prob{}{d_\bX + \text{Lap}(1/\eps) < \Bigg ( \dfrac{t-1}{4} - \dfrac{\log (1/\delta)}{\eps}\Bigg ) \Bigg | \bX \in G } \numberthis \label{eqn: split}
 = \Prob{}{d_\bX + L < \left ( \dfrac{t-1}{4} - \dfrac{\log (1/\delta)}{\eps}\right )} \numberthis \label{eqn: before_split}
\end{align*}

Now, let us define the set
$\cD = \left \{ \bbx \in \cX^n : d_\bbx < \left ( \dfrac{t-1}{4}\right ) \right \}$. If $\left ( d_\bX + L < \left ( \dfrac{t-1}{4} - \dfrac{\log (1/\delta)}{\eps}\right ) \right )$, then either $\bX \in \cD$, or $ L < - \dfrac{\log (1/\delta)}{\eps}$, or both. Thus,
\begin{equation}
\Prob{}{\mathcal{B}_{\eps,\delta}(\bX, \bba) \neq \bot} \leq \Prob{}{\bX \in \cD} + \Prob{}{ L < - \dfrac{\log (1/\delta)}{\eps}} \label{eqn: split}
\end{equation}
%where the inequality follows from equation \eqref{eqn: before_split}\rynote{I do not think this explanation is needed here}.

From the tail bound of the Laplace distribution,
\begin{equation}
\Prob{}{ L < - \dfrac{\log (1/\delta)}{\eps} } \leq \delta \label{eqn: D_2}
\end{equation}

Next, we will calculate the probability of $\bX \in \cD$. %Let $B_q$ denote a Hamming ball of radius $q$, and $Vol(B_q)$ be the volume $B_q$. Also,
We then assign $s \stackrel{\text{def}}{=} \frac{t-1}{4}$. Notice that as the minimum distance of $C_\bba$ is $t$, the Hamming balls $B_{2s}$ of radius $2s$ around the codewords of $C_\bba$ are disjoint and thus we can bound the volume (defined in \eqref{eq:vol}) of each,
\begin{equation}
|C_\bba| \cdot Vol(B_{2s}) \leq 2^n \label{eqn: ball}
\end{equation}

%Let $Vol(C_\bba,s)$ denote the volume of the Hamming balls of radius $s$ around the codewords of $C_\bba$.
Therefore,
\begin{align*}
|C_\bba| \cdot Vol(B_s) & \leq \dfrac{2^n\cdot Vol(B_s)}{Vol(B_{2s})}  =  2^n\cdot  \dfrac{\sum\limits_{i=0}^{s} {n \choose i} }{\sum\limits_{j=0}^{2s} {n \choose j}} \leq 2^n \cdot \dfrac{s \cdot {n \choose s}}{{n \choose 2s}}  \leq  2^n\cdot \dfrac{s \cdot \left (\dfrac{ne}{s}\right )^s}{\left (\dfrac{n}{2s}\right )^{2s}} = 2^n s \left (\dfrac{4es}{n} \right )^s \numberthis \label{eqn: vol}
\end{align*}
where the first inequality follows from equation \eqref{eqn: ball}, and the last inequality follows as  $\left ( \frac{n}{k} \right )^k \leq {n \choose k} \leq \left ( \frac{ne}{k} \right )^k$  for $k\geq 1$  (from Appendix C.1 in  \cite{CLRS09}).

Thus,
\begin{align*}
\Prob{}{\bX \in \cD} & =  \dfrac{|C_\bba| \cdot Vol(B_s)}{2^n} \leq s \left (\dfrac{4es}{n} \right )^s \numberthis \label{eqn: D_1}
\end{align*}
where the inequality follows from equation \eqref{eqn: vol}.

Hence,
\begin{align*}
\Prob{}{\mathcal{B}_{\eps,\delta}(\bX, \bba) \neq \bot} & \leq s \left (\dfrac{4es}{n} \right )^s + \delta  < s \cdot 2^{-s} + \delta \numberthis \label{eqn: not bot}
\end{align*}
where the first inequality follows from equations \eqref{eqn: split},\eqref{eqn: D_2} and \eqref{eqn: D_1}, and the last inequality follows from the fact that $n> 8es = 2e(t-1)$.

Bounding the term $s \cdot 2^{-s}$ from above, we have
\begin{align*}
s \cdot 2^{-s} & = \dfrac{t-1}{4} \cdot 2^{(1-t)/4} = \dfrac{2 \log (1/\delta)}{\eps} \cdot 2^{- 2 \log (1/\delta)/\eps} = \dfrac{2 \log (1/\delta)}{\eps} \cdot \delta^{2/\eps} \\
&  = ( \delta \log (1/\delta) ) \left ( \dfrac{2}{\eps} \cdot \delta^{(2/\eps)-2}\right )  \delta  \leq \delta \numberthis \label{eqn: s_bound}
\end{align*}
where the inequality follows as $\delta \log (1/\delta) \leq 1$ for $\delta \in \left (0, \frac{1}{4} \right ]$, and $\dfrac{2}{\eps}\cdot \delta^{(2/\eps)-2} \leq 1$ for $\eps \in \left (0, \frac{1}{2}\right ], \delta \in \left (0, \frac{1}{4}\right ]$.
From equations \eqref{eqn: not bot} and \eqref{eqn: s_bound},
\begin{equation}
\Prob{}{\mathcal{B}_{\eps,\delta}(\bX, \bba) = \bot} > 1 - 2\delta \label{eqn: bot}
\end{equation}

Now, for any $\bbx \in  \cX^n$,
\begin{align*}
\log \left ( \dfrac{\Prob{}{(\bX,\mathcal{B}_{\eps,\delta}(\bX, \bba))= (\bbx, \bot)}}{\Prob{}{\bX \otimes \mathcal{B}_{\eps,\delta}(\bX, \bba) = (\bbx, \bot)}} \right )  &=\log \left ( \dfrac{\Prob{}{\mathcal{B}_{\eps,\delta}(\bX, \bba) = \bot | \bX = \bbx}}{\Prob{}{\mathcal{B}_{\eps,\delta}(\bX, \bba) = \bot} }\right)  < \log \left (\dfrac{1}{1 - 2 \delta} \right ) \\
& \leq \log \left (\dfrac{1}{1 - 0.5} \right )  = 1 \numberthis \label{eqn: ln_bound}
\end{align*}
where the first inequality follows from equation \eqref{eqn: bot}, and the second inequality follows from the fact that $\delta \leq \frac{1}{4}$.

We then apply Lemma \ref{lem:boundmaxinfo} using \eqref{eqn: bot} and \eqref{eqn: ln_bound} to get,
\begin{equation*}
I_{\infty}^{\beta}(\bX;\mathcal{B}_{\eps,\delta}(\bX, \bba)) \leq 1 \text{, for } \beta \geq 2\delta.
\end{equation*}
\end{proof}

\begin{proof}[Proof of \Cref{thm: naive}, part 3.]
Let us look at the outcome of $\mathcal{B}_{\eps,\delta}(\bbx, \mathcal{A}(\bbx))$. First, as $\bbx \in C_{\mathcal{A}(\bbx)}$, $f(\bbx) = \bbx$ and $d_\bbx=0$. Thus, $\mathcal{B}_{\eps,\delta}(\bbx, \mathcal{A}(\bbx))$ will either return $\bbx$ or $\bot$. Furthermore, we can show the probability of outputting $\bbx$ is high:
\begin{align*}
\Prob{\substack{\text{coins}\\\text{of } \mathcal{B}_{\eps,\delta}}}{\mathcal{B}_{\eps,\delta}(\bbx, \mathcal{A}(\bbx)) = \bbx} = \Prob{}{  \hat{d}_\bbx < w } \geq \Prob{}{ \hat{d}_\bbx < \dfrac{\log (1/\delta)}{\eps}} = \Prob{}{  \text{Lap}(1/\eps) < \dfrac{\log (1/\delta)}{\eps} }  \geq 1 - \delta
\end{align*}
where the first inequality follows from the fact that $ \left ( \dfrac{t-1}{4} - \dfrac{\log (1/\delta)}{\eps}\right ) \geq \dfrac{\log (1/\delta)}{\eps}$, the equality after it follows since $d_\bbx = 0$, and the last inequality follows from a tail bound of the Laplace distribution.
Thus, for every $\bbx \in \cX^n$,
\begin{align}
 \Prob{}{\mathcal{B}_{\eps,\delta}(\bbx, \mathcal{A}(\bbx)) = \bbx} & \geq 1 - \delta. \label{eqn: equal to x}
\end{align}
%$$\Rightarrow \Pr [\mathcal{B}_{\eps,\delta}(x, \mathcal{A}(x)) = \bot] \leq \delta.$$
Consider the event $\cD_\cX = \{ (\bbx,\bbx) : \bbx \in \cX^n \} $. From equation \eqref{eqn: equal to x},

\begin{align}
\Prob{}{ (\bX, \mathcal{B}_{\eps,\delta}(\bX, \mathcal{A}(\bX)) \in \cD_\cX} \geq 1 - \delta. \label{eqn: x dep}
\end{align}
Also, for $\bbb \in \cY $, if $\bbb = \bot$, then $ \Prob{}{\bX = \bbb} = 0$, and if $\bbb \in \cX^n$, then $\Prob{}{\bX=\bbb} = 2^{-n}$ as $\bX$ is drawn uniformly over $ \cX^n$. Thus, for all $\bbb\in \cY$,
\begin{align*}
\Prob{}{ (\bX, \bbb) \in \cD_\cX} \leq 2^{-n}.
\end{align*}
Hence,
\begin{align*}
\Prob{}{ (\bX \otimes \mathcal{B}_{\eps,\delta}(\bX, \mathcal{A}(\bX))) \in \cD_\cX }= \sum_{\bbb \in \cY}\prob{(\bX,\bbb) \in \cD_\cX}\prob{\cB_{\eps,\delta}(\bX,\cA(\bX))= \bbb} \leq 2^{-n} \numberthis \label{eqn: x ran}
\end{align*}
Therefore, for $\beta \leq \dfrac{1}{2} - \delta$,
\begin{align*}
I_{\infty}^{\beta}(\bX;\mathcal{B}_{\eps,\delta}(\bX, \mathcal{A}(\bX))) & =  \log \left( \max\limits_{\substack{\mathcal{O} \subseteq (\mathcal{\bX} \times \mathcal{Y}) , \\ \Prob{}{(\bX, \mathcal{B}_{\eps,\delta}(\bX, \mathcal{A}(X))) \in \mathcal{O} } > \beta }} \dfrac{\Prob{}{(\bX, \mathcal{B}_{\eps,\delta}(\bX, \mathcal{A}(\bX))) \in \mathcal{O} } - \beta}{\Prob{}{(\bX  \otimes \mathcal{B}_{\eps,\delta} (\bX, \mathcal{A}(\bX))) \in \mathcal{O}}} \right) \\
 & \geq \log \left( \dfrac{\Prob{}{(\bX, \mathcal{B}_{\eps,\delta}(\bX, \mathcal{A}(\bX))) \in \cD_\cX } - \beta}{\Prob{}{(\bX  \otimes \mathcal{B}_{\eps,\delta} (\bX, \mathcal{A}(\bX))) \in \cD_\cX}} \right ) \geq \log \left(\dfrac{1 - \delta - \beta}{2^{-n}} \right) \\
 & = n + \log(1 - \delta - \beta)  \geq n - 1
\end{align*}
where the first inequality follows from equation \eqref{eqn: x dep} and as $(1-\delta)>\beta$, the second inequality follows from equations \eqref{eqn: x dep} and \eqref{eqn: x ran}, and the last inequality follows from the fact that $\beta \leq \dfrac{1}{2} - \delta$.
\end{proof}
\fi

\section*{Acknowlegments}

\ifnum\focs=1
R.R. acknowledges support in part by a grant from the Sloan foundation, and NSF grant CNS-1253345. A.R. acknowledges support  in part by a grant from the Sloan foundation, a Google Faculty Research Award, and NSF grants CNS-1513694 and CNS-1253345. O.T. and A.S. acknowledge support  in part by a grant from the Sloan foundation, a Google Faculty Research Award, and NSF grant IIS-1447700.  
\fi
We thank Salil Vadhan for pointing out that our max-information bound can be used to bound mutual information, thus improving on a result in \cite{MMPRTV11}.

%\input{pvalues}
%\clearpage
% \newpage

{\small
\ifnum\focs=1
\bibliographystyle{IEEEtran}
\else
\bibliographystyle{alpha}
\fi
\bibliography{thmrefs}

\newcommand{\etalchar}[1]{$^{#1}$}
\begin{thebibliography}{DFH{\etalchar{+}}15b}

\bibitem[BBB{\etalchar{+}}13]{BBBZZ13}
Richard Berk, Lawrence Brown, Andreas Buja, Kai Zhang, and Linda Zhao.
\newblock Valid post-selection inference.
\newblock {\em The Annals of Statistics}, 41(2):802--837, 2013.

\bibitem[BNS{\etalchar{+}}16]{BNSSSU15}
Raef Bassily, Kobbi Nissim, Adam~D. Smith, Thomas Steinke, Uri Stemmer, and
  Jonathan Ullman.
\newblock Algorithmic stability for adaptive data analysis.
\newblock In {\em Proceedings of the 48th Annual {ACM} on Symposium on Theory
  of Computing, {STOC}}, 2016.

\bibitem[CLN{\etalchar{+}}16]{CLNRW16}
Rachel Cummings, Katrina Ligett, Kobbi Nissim, Aaron Roth, and Zhiwei~Steven
  Wu.
\newblock Adaptive learning with robust generalization guarantees.
\newblock {\em arXiv preprint arXiv:1602.07726}, 2016.

\bibitem[CLRS09]{CLRS09}
Thomas~H. Cormen, Charles~E. Leiserson, Ronald~L. Rivest, and Clifford Stein.
\newblock {\em Introduction to Algorithms, Third Edition}.
\newblock The MIT Press, 3rd edition, 2009.

\bibitem[De12]{De12}
Anindya De.
\newblock Lower bounds in differential privacy.
\newblock In {\em Proceedings of the 9th International Conference on Theory of
  Cryptography}, TCC'12, pages 321--338, Berlin, Heidelberg, 2012.
  Springer-Verlag.

\bibitem[DFH{\etalchar{+}}15a]{DFHPRR15NIPS}
Cynthia Dwork, Vitaly Feldman, Moritz Hardt, Toni Pitassi, Omer Reingold, and
  Aaron Roth.
\newblock Generalization in adaptive data analysis and holdout reuse.
\newblock In C.~Cortes, N.D. Lawrence, D.D. Lee, M.~Sugiyama, R.~Garnett, and
  R.~Garnett, editors, {\em Advances in Neural Information Processing Systems
  28}, pages 2341--2349. Curran Associates, Inc., 2015.

\bibitem[DFH{\etalchar{+}}15b]{DFHPRR15STOC}
Cynthia Dwork, Vitaly Feldman, Moritz Hardt, Toniann Pitassi, Omer Reingold,
  and Aaron~Leon Roth.
\newblock Preserving statistical validity in adaptive data analysis.
\newblock In {\em Proceedings of the Forty-Seventh Annual ACM on Symposium on
  Theory of Computing}, STOC '15, pages 117--126, New York, NY, USA, 2015. ACM.

\bibitem[DKM{\etalchar{+}}06]{DKMMN06}
Cynthia Dwork, Krishnaram Kenthapadi, Frank McSherry, Ilya Mironov, and Moni
  Naor.
\newblock Our data, ourselves: Privacy via distributed noise generation.
\newblock In {\em Advances in Cryptology - {EUROCRYPT} 2006, 25th Annual
  International Conference on the Theory and Applications of Cryptographic
  Techniques, St. Petersburg, Russia, May 28 - June 1, 2006, Proceedings},
  pages 486--503, 2006.

\bibitem[DMNS06]{DMNS06}
Cynthia Dwork, Frank Mcsherry, Kobbi Nissim, and Adam Smith.
\newblock Calibrating noise to sensitivity in private data analysis.
\newblock In {\em In Proceedings of the 3rd Theory of Cryptography Conference},
  pages 265--284. Springer, 2006.

\bibitem[DR14]{DR14}
Cynthia Dwork and Aaron Roth.
\newblock The algorithmic foundations of differential privacy.
\newblock {\em Foundations and Trends in Theoretical Computer Science},
  9(3-4):211--407, 2014.

\bibitem[DRV10]{DRV10}
Cynthia Dwork, Guy~N. Rothblum, and Salil~P. Vadhan.
\newblock Boosting and differential privacy.
\newblock In {\em 51th Annual {IEEE} Symposium on Foundations of Computer
  Science, {FOCS} 2010, October 23-26, 2010, Las Vegas, Nevada, {USA}}, pages
  51--60, 2010.

\bibitem[DSZ15]{DSZ15}
Cynthia Dwork, Weijie Su, and Li~Zhang.
\newblock Private false discovery rate control.
\newblock {\em arXiv preprint arXiv:1511.03803}, 2015.

\bibitem[FST14]{FST14}
William Fithian, Dennis Sun, and Jonathan Taylor.
\newblock Optimal inference after model selection.
\newblock {\em arXiv preprint arXiv:1410.2597}, 2014.

\bibitem[GL14]{GL14}
Andrew Gelman and Eric Loken.
\newblock The statistical crisis in science.
\newblock {\em American Scientist}, 102(6):460, 2014.

\bibitem[GLRV16]{GLRV16}
Marco Gaboardi, Hyun Lim, Ryan Rogers, and Salil Vadhan.
\newblock Differentially private chi-squared hypothesis testing: Goodness of
  fit and independence testing.
\newblock {\em arXiv preprint arXiv:1602.03090}, 2016.

\bibitem[HU14]{HU14}
Moritz Hardt and Jonathan Ullman.
\newblock Preventing false discovery in interactive data analysis is hard.
\newblock In {\em Foundations of Computer Science (FOCS), 2014 IEEE 55th Annual
  Symposium on}, pages 454--463. IEEE, 2014.

\bibitem[JS13]{JS13}
Aaron Johnson and Vitaly Shmatikov.
\newblock Privacy-preserving data exploration in genome-wide association
  studies.
\newblock In {\em Proceedings of the 19th ACM SIGKDD International Conference
  on Knowledge Discovery and Data Mining}, KDD '13, pages 1079--1087, New York,
  NY, USA, 2013. ACM.

\bibitem[KS14]{KS14}
S.P. {Kasiviswanathan} and A.~{Smith}.
\newblock {On the `Semantics' of Differential Privacy: A Bayesian Formulation}.
\newblock {\em Journal of Privacy and Confidentiality}, Vol. 6: Iss. 1, Article
  1, 2014.
\newblock Available at \url{http://repository.cmu.edu/jpc/vol6/iss1/1}. The
  theorem numbers and exact statements refer to the Arxiv version (v3).

\bibitem[KS16]{KS16}
Vishesh Karwa and Aleksandra Slavkovi{\'c}.
\newblock Inference using noisy degrees: Differentially private beta-model and
  synthetic graphs.
\newblock {\em The Annals of Statistics}, 44(1):87--112, 2016.

\bibitem[Lin99]{L99}
Jacobus Hendricus~van Lint.
\newblock {\em Introduction to Coding Theory}.
\newblock Springer, Berlin, 3rd edition, 1999.

\bibitem[LSST13]{LSST13}
Jason~D Lee, Dennis~L Sun, Yuekai Sun, and Jonathan~E Taylor.
\newblock Exact post-selection inference, with application to the lasso.
\newblock {\em arXiv preprint arXiv:1311.6238}, 2013.

\bibitem[MMP{\etalchar{+}}11]{MMPRTV11}
Andrew McGregor, Ilya Mironov, Toniann Pitassi, Omer Reingold, Kunal Talwar,
  and Salil~P. Vadhan.
\newblock The limits of two-party differential privacy.
\newblock {\em Electronic Colloquium on Computational Complexity {(ECCC)}},
  18:106, 2011.

\bibitem[RR10]{RR10}
Aaron Roth and Tim Roughgarden.
\newblock Interactive privacy via the median mechanism.
\newblock In {\em Proceedings of the forty-second ACM symposium on Theory of
  computing}, pages 765--774. ACM, 2010.

\bibitem[RRST16]{RogersRST16}
Ryan Rogers, Aaron Roth, Adam Smith, and Om~Thakkar.
\newblock Max-information, differential privacy, and post-selection hypothesis
  testing.
\newblock In {\em Foundations of Computer Science (FOCS)}, 2016.
\newblock arXiv:1604.03924 [cs.LG].

\bibitem[RZ16]{RZ15}
Daniel Russo and James Zou.
\newblock Controlling bias in adaptive data analysis using information theory.
\newblock In {\em Proceedings of the 19th International Conference on
  Artificial Intelligence and Statistics, {AISTATS}}, 2016.

\bibitem[She15]{Shef15}
Or~Sheffet.
\newblock Differentially private least squares: Estimation, confidence and
  rejecting the null hypothesis.
\newblock {\em arXiv preprint arXiv:1507.02482}, 2015.

\bibitem[SNS11]{SNS11}
J.~P. Simmons, L.~D. Nelson, and U.~Simonsohn.
\newblock {False-Positive Psychology: Undisclosed Flexibility in Data
  Collection and Analysis Allows Presenting Anything as Significant}.
\newblock {\em Psychological Science}, October 2011.

\bibitem[ST13]{ST13}
Adam Smith and Abhradeep Thakurta.
\newblock Differentially private feature selection via stability arguments, and
  the robustness of the lasso.
\newblock In {\em {COLT} 2013 - The 26th Annual Conference on Learning Theory,
  June 12-14, 2013, Princeton University, NJ, {USA}}, pages 819--850, 2013.

\bibitem[SU15]{SU15}
Thomas Steinke and Jonathan Ullman.
\newblock Interactive fingerprinting codes and the hardness of preventing false
  discovery.
\newblock In {\em Proceedings of The 28th Conference on Learning Theory}, pages
  1588--1628, 2015.

\bibitem[USF13]{USF13}
Caroline Uhler, Aleksandra Slavkovic, and Stephen~E. Fienberg.
\newblock Privacy-preserving data sharing for genome-wide association studies.
\newblock {\em Journal of Privacy and Confidentiality}, 5(1), 2013.

\bibitem[WL16]{WL16}
Ronald~L. Wasserstein and Nicole~A. Lazar.
\newblock The asa's statement on p-values: context, process, and purpose.
\newblock {\em The American Statistician}, 0(ja):00--00, 2016.

\bibitem[WLF16]{WLF16}
Yu{-}Xiang Wang, Jing Lei, and Stephen~E. Fienberg.
\newblock A minimax theory for adaptive data analysis.
\newblock {\em CoRR}, abs/1602.04287, 2016.

\bibitem[WLK15]{WLK15}
Yue Wang, Jaewoo Lee, and Daniel Kifer.
\newblock Differentially private hypothesis testing, revisited.
\newblock {\em arXiv preprint arXiv:1511.03376}, 2015.

\bibitem[YFSU14]{YFSU14}
Fei Yu, Stephen~E. Fienberg, Aleksandra~B. Slavkovic, and Caroline Uhler.
\newblock Scalable privacy-preserving data sharing methodology for genome-wide
  association studies.
\newblock {\em Journal of Biomedical Informatics}, 50:133--141, 2014.

\end{thebibliography}
}

\ifnum\focs=0
\appendix
\section{Other Useful Probabilistic Tools} \label{app_probtools}
We use this section to give an overview of some useful probabilistic tools and their connections with one another.  We start by giving a commonly used differentially private mechanism, called the \emph{Laplace Mechanism}, which releases an answer to a query on the dataset with appropriately scaled Laplace noise.  We use this mechanism in the proof of \Cref{thm: naive}.

\begin{theorem}[Laplace Mechanism \cite{DMNS06}]
Let $\phi: \cX^n \rightarrow \mathbb{R}$ be a function such that for any pair of points $\bbx$ and $\bbx'$ where  $dist_{Hamm}(\bbx,\bbx') = 1$, then we have $|\phi(\bbx) - \phi(\bbx')| \leq 1$.  The mechanism $\cM(\bbx) = \phi(\bbx) + L$ where $L \sim \text{Lap}(1/\eps)$, is $(\eps,0)$-differentially private.
\label{thm:lap}
\end{theorem}
A very useful fact about differentially private mechanisms is that one cannot take the output of a differentially private mechanism and perform any modification to it that does not depend on the input itself and make the output any less private.
\begin{theorem}[Post Processing \cite{DMNS06}]
\label{thm:pp}
Let $\cM: \cX^n \to \cY$ be $(\epsilon,\delta)$-differentially private and $\psi: \cY \to \cY'$ be any function mapping to arbitrary domain $\cY'$.  Then $\psi \circ \cM$ is $(\epsilon,\delta)$-differentially private.
\label{thm: post_proc}
\end{theorem}

We have been focusing on approximate max-information throughout this paper, which has the following strong composition guarantee:
\begin{theorem}[Composition \citep{DFHPRR15NIPS}]
Let $\cA_1:\cX^n\to\cY$ and $\cA_2: \cX^n \times \cY \to \cZ$ be such that $I_\infty^{\beta_1}(\cA_1,n) \leq k_1$ and $I_\infty^{\beta_2}(\cA_2(\cdot,y),n) \leq k_2$ for every $y \in \cY$.  Then the composition of $\cA_1$ and $\cA_2$, defined to be $\cA(\bX) = \cA_2(\bX,\cA_1(\bX))$ satisfies:
$$
I_\infty^{\beta_1+\beta_2}(\cA,n) \leq k_1 + k_2.
$$
\label{lem:maxinfocomp}
\end{theorem}

In our analysis, in addition to max-information, we will use more familiar measures between two random variables, which we give here:
\begin{defn}[KL Divergence] \label{defn: KL}
The KL Divergence between random variables $X$ and $Z$, denoted as $D_{KL}(X|| Z)$ over domain $\cD$ is defined as
$$
D_{KL}(X|| Z)=\sum_{x \in \cD} \Prob{}{X = x} \ln\left( \frac{\Prob{}{X=x}} {\Prob{}{Z = x} } \right)
$$
\end{defn}
\begin{defn}[Total Variation Distance]
The \emph{total variation distance} between two random variables $X$ and $Z$, denoted as $TV(X;Z)$, over domain $\cD$ is defined as
$$
TV(X,Z) = \frac{1}{2} \cdot \sum_{x \in \cD} | \Prob{}{X = x} - \Prob{}{Z = x} |.
$$
\end{defn}

In the following lemma, we state some basic connections between max-information, total variation distance, differential privacy, and indistinguishability:
\begin{lemma} \label{lem:prelims}
 Let $X, Z$ be two random variables over the same domain.  We then have:
\begin{enumerate}
\item \cite{DFHPRR15NIPS} $I_\infty^\beta (X;Z) \leq k \Leftrightarrow (X,Z) \approx_{\left(k\ln2\right),\beta} X \otimes Z$.
%\item $X \edi Z \implies TV(X;Z) \leq e^\eps - 1 + \delta$.
\item \citep{KS14} If $X \edi Y$ then $X$ and $Y$ are pointwise $\left(2\eps, \frac{2\delta}{1-e^{-\eps}} \right) $-indistinguishable.
%\item \citep{KS14} Let $\cA : \cX^n \to \cY$ be a randomized algorithm where $\bbx,\bbx' \in \cX^n$ are two datasets that differ in at most one entry.  We then have that $\cA$ is $(\eps,\delta)$-differentially private $\Leftrightarrow \cA(\bbx) \edi \cA(\bbx')$
\end{enumerate}
\end{lemma}

Another useful result is from \cite{KS14}, which we use in the proof of our main result in \Cref{thm:main}:
\begin{lemma}[Conditioning Lemma]\label{lem:conditioning}
Suppose that $(X,Z) \edi (X',Z')$.  Then for every $\hat \delta>0$, the following holds:
$$
\Prob{t \sim p(Z) }{ X|_{Z = t} \approx_{3\epsilon, \hat\delta} X'|_{Z'=t} } \geq 1-\frac{2\delta}{\hat\delta} - \frac{2\delta}{1-e^{-\eps}}.
$$
\end{lemma}

The proof of our main result in \Cref{thm:main} also makes use of the following standard concentration inequality:
\begin{theorem}[Azuma's Inequality]
Let $C_1, \cdots, C_n$ be a sequence of random variables such that for every $i \in [n]$, we have
$$
\Prob{}{|C_i| \leq \alpha} = 1
$$
and for every fixed prefix $\bC_1^{i-1} = \bbc_1^{i-1}$, we have
$$
\Ex{}{C_i|\bbc_1^{i-1}} \leq \gamma,
$$
then for all $t\geq 0$, we have
$$
\Prob{}{\sum_{i=1}^nC_i > n \gamma + t \sqrt{n} \alpha} \leq e^{-t^2/2}.
$$
\label{thm:azuma}
\end{theorem}

We say that a real-valued function $f: \cX^n \to \R$ has sensitivity $\Delta$ if for all $i \in [n]$ and $\bbx \in \cX^n$ and $x_i' \in \cX, \left| f(\bbx_{-i}, x_i) - f(\bbx_{-i}, x_i') \right| \leq \Delta$. One important use of max-information bounds is that they imply strong generalization bounds for low sensitivity functions when paired with McDiarmid's inequality.
\begin{theorem}[McDiarmid's Inequality]
Let $X_1,\cdots, X_n$ be independent random variables with domain $\cX$.  Further, let $f: \cX^n \to \R$ be a function of sensitivity $\Delta>0$.  Then for every $\tau>0$ and $\mu = \Ex{}{f(X_1,\cdots, X_n)}$ we have
$$
\Prob{}{f(X_1,\cdots, X_n) - \mu \geq \tau} \leq \exp\left( \frac{-2\tau^2}{n \Delta^2}  \right).
$$
\label{thm:McD}
\end{theorem}

\section{Sensitivity of $p$-values}\label{sec:sens_p}
In this section, we demonstrate that the theorem from \cite{BNSSSU15}, which shows that differentially private algorithms which select low-sensitivity queries cannot overfit, does not give nontrivial generalization bounds for $p$-values.  We first show that $p$-values cannot have sensitivity smaller than $.37/\sqrt{n}$.
\begin{lemma}
\label{lem:p_lb}
Let $\phi:\cX^n \to \R$ be a test statistic with null hypothesis $H_0$, and $p: \R \to [0,1]$, where $p(a) = \Prob{\bbx \sim \cP^n}{\phi(\bbx) \geq a}$, and $\cP \in H_0$.  The sensitivity of $p \circ \phi$ must be larger than $0.37/\sqrt{n}$.
\end{lemma}
\begin{proof}
Note that if $\bX \sim \cP^n$, then $p \circ \phi(\bX)$ is uniform on $[0,1]$, and thus, has mean $1/2$.  From \Cref{thm:McD}, we know that if $p \circ \phi$ has sensitivity $\Delta$, then for any $0<\delta < 1/2$, we have:
$$
\Prob{}{p \circ \phi (\bX) \geq 1/2+ \Delta\sqrt{\frac{n}{2} \ln(1/\delta)}} \leq \delta.
$$
However, we also know that $p\circ\phi(\bX)$ is uniform, so that
$$
\Prob{}{p\circ\phi(\bX) \geq 1-\delta} = \delta.
$$
Hence, if $\Delta < \frac{1/2 - \delta}{\sqrt{\frac{n}{2}\ln(1/\delta)}}$, we obtain a contradiction:
$$
\delta \geq \Prob{}{p \circ \phi (\bX) \geq 1/2+ \Delta\sqrt{\frac{n}{2} \ln(1/\delta)}} > \Prob{}{p\circ\phi(\bX) \geq 1-\delta} = \delta.
$$
We then set $\delta = 0.08$ to get our stated bound on sensitivity.
\end{proof}

Thus, the sensitivity $\Delta$ for the $p$-value for any test statistic and any null hypothesis must be at least $0.37/\sqrt{n}$.  This is too large for the following theorem, proven in \cite{BNSSSU15}, to give a nontrivial guarantee:
\begin{theorem}[\cite{BNSSSU15}]
Let $\epsilon \in (0,1/3)$, $\delta \in (0,\epsilon/4)$, and $n \geq \frac{\ln(4\epsilon/\delta)}{\epsilon^2}$.  Let $\cY$ denote the class of $\Delta$-sensitive functions $f:\cX^n\rightarrow \R$, and let $\cA: \cX^n \to \cY$ be an algorithm  that is $(\epsilon,\delta)$-differentially private.  Let $\bX \sim \cP^n$ for some distribution $\cP$ over $\cX$, and let $q = \cA(\bX)$.  Then:
$$
\Prob{\bX,\cA}{|q(\cP^n) - q(\bX)| \geq 18 \epsilon \Delta n} < \frac{\delta}{\epsilon}.
$$
\label{thm:bns}
\end{theorem}

When we attempt to apply this theorem to a $p$-value, we see that the error it guarantees, by \Cref{lem:p_lb}, is at least: $18 \epsilon \Delta n \geq 18 \epsilon (.37) \sqrt{n}$. However, the theorem is only valid for $n \geq \frac{1}{\eps^2} \ln\left(\frac{4\eps}{\delta} \right).$ Plugging this in, we see that $18 \epsilon (.37) \sqrt{n} \geq 1$, which is a trivial error guarantee for $p$-values (which take values in $[0,1]$).

\section{General Conversion between Max-Information and Mutual Information, and its Consequences}\label{sect:mutual-max}
\rynote{Modified Whole Section}
We give here a proof of \Cref{lem:mutual-max}.
\begin{proof}[Proof of \Cref{lem:mutual-max}]
We first prove the first direction of the lemma.
We define a set $G(k) = \{(\bbx, \bby) : Z(\bbx, \bby) \leq k \}$ where $Z(\bbx, \bby) = \log\left( \frac{\prob{(\bX, \bY) = (\bbx, \bby)}}{\prob{\bX \otimes \bY = (\bbx, \bby)} } \right)$, and $\beta(k)$ to be the quantity such that $\prob{(\bX,\bY) \in G(k)} = 1-\beta(k)$.  We then have:
\begin{align*}
m & \geq \sum_{(\bbx, \bby) \in G(k)} \prob{(\bX,\bY) = (\bbx, \bby)} Z(\bbx, \bby)  + \sum_{(\bbx, \bby) \notin G(k)} \prob{(\bX,\bY) = (\bbx, \bby)} Z(\bbx, \bby) \\
& \geq \sum_{(\bbx, \bby) \in G(k)} \prob{(\bX,\bY) = (\bbx, \bby)} Z(\bbx, \bby) + k \beta(k) \numberthis \label{eqn: m_int_bound}
\end{align*}
%Also, we can apply Jensen's inequality to get:
Also, we have:
\begin{align*}
- \sum_{(\bbx, \bby) \in G(k)} & \frac{\prob{(\bX,\bY) = (\bbx, \bby)} }{\prob{(\bX,\bY) \in G(k)}} Z(\bbx, \bby) \leq \log\left( \frac{\prob{\bX\otimes\bY \in G(k)} }{\prob{(\bX,\bY) \in G(k)} } \right) \leq \log\left( \frac{1}{1-\beta(k)} \right)
\end{align*}
where the first inequality follows from applying Jensen's inequality.

By rearranging terms, we obtain:
$$ \sum_{(\bbx, \bby) \in G(k)} \prob{(\bX, \bY) = (\bbx, \bby)} \cdot Z(\bbx, \bby) \geq -(1-\beta(k))\cdot \log\left( \frac{1}{1-\beta(k)} \right).$$
Plugging the above in \Cref{eqn: m_int_bound}, we obtain:
\begin{align*}
m \geq -(1-\beta(k) ) \log\left( \frac{1}{1-\beta(k)}\right) + k \beta(k).
\end{align*}
Thus,
\begin{align*}
 k \leq \frac{m + (1-\beta(k) ) \log\left(\frac{1}{1-\beta(k)} \right) }{\beta(k)} \leq \frac{m + 0.54}{\beta(k)}
\end{align*}
where the last inequality follows from the fact that the function $(1-w) \log(1/(1-w))$ is maximized at $w = (e-1)/e$ (and takes value $< 0.54$). Solving for $\beta(k)$ gives the claimed bound.

We now prove the second direction of the lemma.  Note that we can use \Cref{lem:prelims} to say that $(\bX,\bY)$ and $(\bX \otimes \bY)$ are point-wise $\left(2k \ln(2) ,\beta'\right)$ indistinguishable, for $\beta' =  \frac{2\beta}{1- 2^{-k} }$. Note that $\beta' \leq 0.3$ as $\beta \in \left[0,\frac{3(1-2^{-k})}{20}\right]$.  We now define the following ``bad" set $\cB$:
$$
\cB = \left\{(\bbx, \bby) \in \Sigma \times \Sigma : \ln\left(\frac{(\bX,\bY) = (\bbx, \bby)}{\bX\otimes \bY = (\bbx, \bby)} \right) > 2k\ln(2) \right\}.
$$
From the definition of point-wise indistinguishability, we have:
\begin{align}
I(\bX;\bY) \leq 2k \ln(2) + \sum_{(\bbx, \bby) \in \cB}\prob{(\bX, \bY) = (\bbx, \bby)}  \ln\left(\frac{\prob{(\bX, \bY) = (\bbx, \bby)}}{\prob{\bX = \bbx}\prob{\bY = \bby} }\right).
\label{eq:maxmut1}
\end{align}
Now, if set $\cB$ is empty, we get from \eqref{eq:maxmut1} that:
\begin{equation*}
I(\bX;\bY) \leq 2k \ln(2)
\end{equation*}
which trivially gives us the claimed bound.

However, if set $\cB$ is non-empty, we then consider the mutual information conditioned on the event $(\bX, \bY) \in \cB$.  We will write $\bX'$ for the random variable $\bX$ conditioned on $(\bX, \bY) \in \cB$, and $\bY'$ for the random variable $\bY$ conditioned on the same event.  We can then obtain the following bound, where we write $H(W)$ to denote the entropy of random variable $W$:
$$
I(\bX';\bY') \leq H(\bX') \leq \log |\Sigma|.
$$
We can then make a relation between the mutual information $I(\bX',\bY')$ and the sum of terms over $\cB$ in \eqref{eq:maxmut1}.  Note that, for $(\bbx, \bby) \in \cB$, we have from Bayes' rule:
$$\prob{(\bX,\bY) \in \cB} \prob{(\bX,\bY) = (\bbx, \bby) | \cB} =\prob{(\bX, \bY) = (\bbx, \bby)}.$$
This then gives us the following bound:
\begin{align}
\sum_{(\bbx, \bby) \in \cB} \frac{\prob{(\bX, \bY) = (\bbx,\bby)} }{\prob{(\bX,\bY) \in \cB}} \ln\left(\frac{\prob{(\bX, \bY) = (\bbx,\bby)}}{ \prob{(\bX,\bY) \in \cB}\prob{\bX = \bbx | \cB}\prob{\bY = \bby|\cB} }\right) \leq \log |\Sigma| .
\label{eq:maxmut2}
\end{align}
Note that $\prob{\bX = \bbx | \cB} \leq \frac{\prob{\bX = \bbx}}{\prob{(\bX,\bY)\in \cB}}$, and similarly, $\prob{\bY = \bby | \cB} \leq \frac{\prob{\bY = \bby}}{\prob{(\bX,\bY)\in \cB}}$.  From \eqref{eq:maxmut2}, we have:
\begin{align*}
\sum_{(\bbx, \bby) \in \cB}& \prob{(\bX, \bY) = (\bbx,\bby)}  \ln\left(\frac{\prob{(\bX, \bY) = (\bbx,\bby)}}{\prob{\bX = \bbx}\prob{\bY = \bby} }\right) \\
& \qquad \leq \prob{(\bX,\bY) \in \cB} \log |\Sigma| + \prob{(\bX, \bY) \in \cB} \ln\left(1/\prob{(\bX,\bY) \in \cB}\right) \\
& \qquad \leq \beta' (\log |\Sigma| + \ln(1/\beta')) \\
& \qquad = \frac{2\beta }{1- 2^{-k} } \left( \log |\Sigma|+  \ln\left(\frac{1- 2^{-k} }{2\beta}\right)\right) \\
& \qquad \leq \frac{2\beta \log (|\Sigma|/2 \beta)  }{1- 2^{-k} }
\end{align*}
where the last inequality follows from the fact that $\prob{(\bX,\bY) \in \cB} \leq \beta'$ and $w \ln(1/w) \leq \beta' \ln(1/\beta')$ when $0\leq w \leq \beta' \leq 0.3$. \omnote{Added later.} Substituting the sum of terms over $\cB$ in \eqref{eq:maxmut1} by the above gives us the claimed bound.

\end{proof}

We have already shown how \Cref{lem:mutual-max} along with \Cref{thm:main} can be used to improve a bound from \cite{MMPRTV11}.  We now state another corollary which improves on a result from \cite{RZ15}.

In the introduction, we proved a simple theorem about how to correct $p$-values (using a \emph{valid $p$-value correction function} -- \Cref{defn:p_correct}) given a bound on the \emph{max-information} between the input dataset and the test-statistic selection procedure $\cA$ (Theorem \ref{thm:maxinfo-pvalues}).  We note that we can easily extend the definition of null hypotheses given in the introduction (and hence $p$-values and correction functions), to allow for distributions $\cS$ over $\cX^n$ that need not be product distributions.  In fact, we can restate \Cref{thm:maxinfo-pvalues} in terms of non-product distributions:

\begin{theorem}
\label{thm:maxinfo-pvalues_nonproduct}
Let $\cA:\cX^n\rightarrow \cT$ be a data-dependent algorithm for selecting a test statistic such that $I^\beta_{\infty}(\cA, n) \leq k$. Then the following function $\gamma$ is a valid $p$-value correction function for $\cA$:
$$\gamma(\alpha) = \max\left(\frac{\alpha - \beta}{2^k},0\right).$$
\end{theorem}
\begin{proof}
The proof is exactly the same as the proof of \Cref{thm:maxinfo-pvalues}, except we fix an arbitrary (perhaps non-product) distribution $\cS$ from which the dataset $\bX$ is drawn.
%If $\frac{\alpha - \beta}{2^k} \leq 0$, then the theorem is trivial, so assume otherwise.
%Define $\cO \subset \cX^n \times \cT$ to be the event that $\cA$ selects a test statistic for which the null hypothesis is true, but its $p$-value is at most $\gamma(\alpha)$:
%$$\cO = \{(\bbx, \phi_i) : \cS \in H_0^{(i)} \textrm{ and }p_i(\phi_i(\bbx)) \leq \gamma(\alpha)\}$$
%Note that the event $\cO$ represents exactly those outcomes for which using $\gamma$ as a $p$-value correction function results in a false discovery.
%Note also that, by definition of the null hypothesis, $\Pr[\bX \otimes \cA(\bX) \in \cO] \leq \gamma(\alpha) = \frac{\alpha-\beta}{2^k}$. Hence, by the guarantee that $I^\beta_{\infty}(\cA, n) \leq k$, we have that
%$\Pr[(\bX, \cA(\bX) \in \cO)] $ is at most $ 2^k\cdot\left(\frac{\alpha-\beta}{2^k}\right) + \beta = \alpha$.
\end{proof}

%Hence, given a test statistic $\phi_i:\cX^n \to \R$ and null hypothesis $H_0^i$, we define $p_i^\cS(a) = \Prob{\bbx \sim \cS}{\phi_i(\bbx) \geq a}$ where $\cS \in H_0^i$.  We will again hide the dependence on distribution $\cS$ and write $p_i(\cdot)$ when we are given a null hypothesis $H_0^i$.  \rynote{better?}

Previously, Russo and Zou \cite{RZ15} have given a method to correct $p$-values given a bound on the \emph{mutual information} between the input data and the test-statistic selection procedure.\footnote{Actually, \cite{RZ15} do not explicitly model the dataset, and instead give a bound in terms of the mutual information between the test-statistics themselves and the test-statistic selection procedure. We could also prove bounds with this dependence, by viewing our input data to be the value of the given test-statistics, however for consistency, we will discuss all bounds in terms of the mutual information between the data and the selection procedure.} In \Cref{lem:mutual-max}, we observe that if we had a bound on the mutual information between the input data and the test-statistic procedure, this would imply a bound on the \emph{max-information} that would be sufficiently strong so that our Theorem \ref{thm:maxinfo-pvalues_nonproduct} would give us the following corollary:
\begin{corollary}
Let $\cA: \cX^n \to \cT$ be a test-statistic selection procedure such that $I(\bX; \cA(\bX)) \leq m$. Then $\gamma(\alpha)$ is a valid $p$-value correction function, where:
$$
\gamma(\alpha) = \frac{\alpha}{2} \cdot 2^{\frac{-2}{\alpha}(m+0.54)}.
$$
\end{corollary}
\begin{proof}
From \Cref{lem:mutual-max}, we know that for any $k>0$, $I_{\infty}^{\beta(k)}(\cA,n) \leq k$, where $\beta(k) \leq \frac{m + 0.54}{k}$.  Hence, from Theorem \ref{thm:maxinfo-pvalues_nonproduct}, we know that for any choice of $k>0$, $\gamma(\alpha)$ is a valid $p$-value correction function where:
$$
\gamma(\alpha) = \frac{\alpha - \frac{m+0.54}{k}}{2^k}.
$$
 Choosing $k = \frac{2(m+0.54)}{\alpha}$ gives our claimed bound.
\end{proof}

We now show that this gives us a strictly improved $p$-value correction function than the bound given by \cite{RZ15}, which we state here using our terminology.

\begin{theorem}[\cite{RZ15} Proposition 7]
Let $\cA:\cX^n\rightarrow \cT$ be a test-statistic selection procedure such that $I(\bX;\cA(\bX)) \leq m$ (where $I$ denotes mutual information). If we define $\phi_i = \cA(\bX)$, then for every $\gamma \in [0,1]$:
$$\Prob{}{p_i(\phi_i(\bX)) \leq \gamma} \leq \gamma + \sqrt{\frac{m}{\ln(1/2\gamma)}}.$$
\end{theorem}

If we want to set parameters so that the probability of a false discovery is at most $\alpha$, then in particular, we must pick $\gamma$ such that $\sqrt{\frac{m}{\ln(1/2\gamma)}} \leq \alpha$. Equivalently, solving for $\alpha$, the best valid $p$-value correction function implied by the bound of \cite{RZ15} must satisfy:
$$\gamma_{RZ}(\alpha) \leq \min\left\{\frac{\alpha}{2}, \frac{1}{2}\cdot 2^{-\log(e)m/\alpha^2} \right\}.$$

We can obtain a better bound by instead arguing via max-information combining \Cref{lem:mutual-max} with Theorem \ref{thm:maxinfo-pvalues_nonproduct}, we can derive a valid $p$-value correction function given a bound on the mutual information between the data and the selection procedure:

Comparing the $p$-value correction function $\gamma(\alpha)$ derived above, with the function $\gamma_{RZ}(\alpha)$ that arises from \cite{RZ15}, we see that the version above has an exponentially improved dependence on $1/\alpha$.  Moreover, it almost always gives a better correction factor in practice: for any value of $\alpha \leq 0.05$, the function $\gamma(\alpha)$ derived above improves over $\gamma_{RZ}(\alpha)$ whenever the mutual information bound $m \geq 0.05$ (whereas, we would naturally expect the mutual information to be $m \gg 1$, and to scale with $n$).

\section{Rederiving Generalization for Low Sensitivity Queries, and a Comparison with the Bounds of \cite{BNSSSU15}}
\label{app:compare}

In this section, we use the bound from our main theorem (\Cref{thm:main}) to rederive known results for the generalization properties of differentially private algorithms which select \emph{low sensitivity queries}. Our bounds do not exactly -- but nearly -- match the tight bounds for this problem, given in \cite{BNSSSU15}. This implies a limit on the extent to which our main theorem can be quantitatively improved, despite its generality.

The goal of generalization bounds for low sensitivity queries is to argue that with high probability, if a low sensitivity function $f:\cX^n\rightarrow \cR = \cA(\bbx)$ is chosen in a data-dependent way when $\bbx$ is sampled from a product distribution $\cP^n$, then the value of the function on the realized data $f(\bbx)$ is close to its expectation $f(\cP^n) \stackrel{\text{def}}{=} \Ex{\bX \sim \cP^n}{f(\bX)}$. If $f$ were selected in a data-independent manner, this would follow from McDiarmid's inequality. No bound like this is true for arbitrary selection procedures $\cA$, but a tight bound by \cite{BNSSSU15} is known when $\cA$ is $(\epsilon,\delta)$-differentially private -- we have already quoted it, as \Cref{thm:bns}.

Using our main theorem (\Cref{thm:main}), together with McDiarmid's inequality (\Cref{thm:McD}), we can derive a comparable statement to \Cref{thm:bns}:
\begin{theorem}
Let $\epsilon \in (0,1)$, $\delta = O(\epsilon^5)$, and $n = \Omega\left( \frac{\log(\epsilon/\delta)}{\epsilon^2}\right)$. Let $\cY$ denote the class of $\Delta$-sensitive functions $f:\cX^n\rightarrow \R$, and let $\cA: \cX^n \to \cY$ be an algorithm  that is $(\epsilon,\delta)$-differentially private.  Let $\bX \sim \cP^n$ for some distribution $\cP$ over $\cX$, and let $q = \cA(\bX)$.  Then there exists a constant $C$ such that:
$$
\Prob{\bX,\cA}{|q(\cP^n) - q(\bX)| \geq C \epsilon \Delta n} < n\sqrt{\frac{\delta}{\epsilon}}
$$
\label{thm:bns2}
\end{theorem}
\begin{proof}
%We apply our approximate max-info bound for $(\epsilon,\delta)$-differentially private algorithms from \Cref{thm:main} along with  McDiarmid's inequality, which gives concentration between the empirical average and the expectation of $q$ when it is not selected as a function of the data to obtain:

If $\cA$ satisfied $I_\infty^{\beta}(\bX; \cA(\bX)) \leq k$, then McDiarmid's inequality (\Cref{thm:McD}) paired with \Cref{defn:maxinfo} would imply:
$$
\Prob{\bX,\cA}{|q(\cP^n) - q(\bX)| \geq C \epsilon \Delta n} < 2^{k} \exp\left( -2 C^2 \epsilon^2 n  \right) + \beta
$$
Because $\cA$ is $(\epsilon,\delta)$-differentially private, \Cref{thm:main} implies that indeed $I_\infty^{\beta}(\bX; \cA(\bX)) \leq k$ for
 $k = O\left(n \eps^2  + n \sqrt{\frac{\delta}{\eps}} \right) = O\left( \epsilon^2 n \right)$ and $\beta = O\left( n \sqrt{\frac{\delta}{\epsilon}}\right) + e^{-\epsilon^2 n} = O\left( n \sqrt{\frac{\delta}{\epsilon}}\right)$, and the claimed bound follows.
\end{proof}

Because \Cref{thm:bns} is asymptotically tight, this implies a limit on the extent to which \Cref{thm:main} can be quantitatively improved.

For comparison, note that the generalization bound for $(\epsilon,\delta)$-differentially private mechanisms given in \Cref{thm:bns} differs by only constants from the generalization bound proven via the max-information approach for $(\epsilon,\delta')$-differentially private mechanisms, where:
$$\delta' = \frac{\delta^2}{\epsilon n^2}$$

Note that in most applications (including the best known mechanism for answering large numbers of low sensitivity queries privately--- the median mechanism \cite{RR10} as analyzed in \cite{DR14} Theorem 5.10), the accuracy of a differentially private algorithm scales with $\sqrt{\frac{\log(1/\delta)}{n}}$ (ignoring other relevant parameters). In such cases, using the bound derived from the max-information approach yields an accuracy that is worse than the bound from \Cref{thm:bns} by an additive term that is $O\left(\sqrt{\frac{\log(\epsilon n)}{n}}\right)$.

\section{Omitted Proofs}
\begin{proof}[Proof of \Cref{cor:nonproduct}]
We use the same algorithms $\cA$ and $\cB$ from \Cref{thm: naive}.  Suppose that for all $\bba \in \{0,1 \}^r$ and $0<\beta_2 < 1/2 - \delta - \beta_1$ for any $\beta_1\in (0,1/2-\delta)$, we have:
$$
I_\infty^{\beta_2}(\cB(\cdot, \bba), n) < n - 1 - r - \log(1/\beta_1) . %\quad \implies \quad I_\infty^{\beta_2+\beta_1}(\cB \circ \cA;n) < n - 1.
$$
Note that because  $\cA$ has bounded description length $r$, we can bound $I_\infty^{\beta_1}(\cA,n) \leq r + \log(1/\beta_1)$ for any $\beta_1>0$ \cite{DFHPRR15NIPS}.  We then apply the composition theorem for max-information mechanisms, \Cref{lem:maxinfocomp}, to obtain:
$$
I_\infty^{\beta_2+\beta_1}(\cB \circ \cA,n) < n - 1.
$$
However, this contradicts \Cref{thm: naive}, because for any $\beta< 1/2 - \delta$,
$$
I_\infty^{\beta}(\cB \circ \cA,n) \geq I_{\infty,P}^\beta(\cB\circ\cA,n) \geq n-1.
$$
Thus, we know that there exists some $\bba^* \in \{ 0,1\}^r$ and (non-product) distribution $\bX \sim \cS$ such that:
$$
I_\infty^{\beta_2}(\bX;\cB(\bX, \bba^*) \geq n - 1 - r - \log(1/\beta_1).
$$
We then define $\cC: \cX^n \to \cY$ to be $\cC(\bbx) = \cB(\bbx,\bba^*)$.  Hence,
$$I_\infty^{\beta_2}(\cC,n) \geq I_\infty^{\beta_2}(\bX;\cC(\bX)) \geq n - 1 - r - \log(1/\beta_1)$$
which completes the proof.
\end{proof}

\fi
\end{document}